\documentclass[a4paper]{article}

\usepackage{amsmath}
\usepackage{amssymb}
\usepackage{mathtools}
\usepackage{amsthm}
\usepackage{enumitem}
\usepackage{xcolor}
\usepackage{soul}
\usepackage{tikz}
\usepackage{url}
\usepackage[subnum]{cases}
\usepackage{bm}
\usepackage[]{lineno}
\usepackage[numbers,sort&compress]{natbib}
\PassOptionsToPackage{algo2e,ruled}{algorithm2e}
\RequirePackage{algorithm2e}
\usepackage{multirow}
\usepackage{hyperref}

\theoremstyle{plain}
\newtheorem{theorem}{Theorem}[section]
\newtheorem{proposition}[theorem]{Proposition}
\newtheorem{lemma}[theorem]{Lemma}
\newtheorem{corollary}[theorem]{Corollary}
\theoremstyle{definition}
\newtheorem{definition}[theorem]{Definition}

\theoremstyle{remark}
\newtheorem{remark}[theorem]{Remark}

\newcommand{\nbb}{\mathbb{N}}

\newcommand{\bw}{\mathbf{w}}
\newcommand{\prox}{\mbox{Prox}}

\newcommand{\ibb}{\mathbb{I}}
\newcommand{\xcal}{\mathcal{X}}
\newcommand{\wcal}{\mathcal{W}}

\newcommand{\zcal}{\mathcal{Z}}
\newcommand{\acal}{\mathcal{A}}
\newcommand{\pbb}{\mathbb{P}}
\newcommand{\mcal}{\mathcal{M}}
\newcommand{\ncal}{\mathcal{N}}

\newcommand{\ycal}{\mathcal{Y}}

\newcommand{\ebb}{\mathbb{E}}

\newcommand{\bv}{\mathbf{v}}

\newcommand{\rbb}{\mathbb{R}}

\numberwithin{equation}{section}

\title{Stability and Generalization of Stochastic Optimization with Nonconvex and Nonsmooth Problems}

\author{%
  Yunwen Lei$^{1}$\\[1.2pt]
  $^1$Department of Mathematics, The University of Hong Kong\\[1.2pt]
  \texttt{leiyw@hku.hk}
}

\begin{document}

\maketitle

\begin{abstract}%
Stochastic optimization has found wide applications in minimizing objective functions in machine learning, which motivates a lot of theoretical studies to understand its practical success. Most of existing studies focus on the convergence of optimization errors, while the generalization analysis of stochastic optimization is much lagging behind. This is especially the case for nonconvex and nonsmooth problems often encountered in practice. In this paper, we initialize a systematic stability and generalization analysis of stochastic optimization on nonconvex and nonsmooth problems. We introduce novel algorithmic stability measures and establish their quantitative connection on the gap between population gradients and empirical gradients, which is then further extended to study the gap between the Moreau envelope of the empirical risk and that of the population risk. To our knowledge, these quantitative connection between stability and generalization in terms of either gradients or Moreau envelopes have not been studied in the literature. We introduce a class of sampling-determined algorithms, for which we develop bounds for three stability measures. Finally, we apply these results to derive error bounds for stochastic gradient descent and its adaptive variant, where we show how to achieve an implicit regularization by tuning the step sizes and the number of iterations.
\end{abstract}

%


\vspace*{-0.210cm}
\section{Introduction}
\vspace*{-0.131cm}

Stochastic optimization has become the workhorse behind many successful applications of machine learning (ML)~\citep{zhang2004solving,bottou2018optimization}. The basic idea is to introduce randomness into the design of optimization algorithms to speed up the learning process by using the sum structure of objective functions in ML. A representative algorithm is the stochastic gradient descent (SGD). As an iterative algorithm, SGD first randomly selects a single example from a training dataset to build a stochastic gradient, and then moves along the negative direction of this stochastic gradient to get the next iterate. Due to its cheap computation cost and simplicity, SGD is especially interesting to solve large-scale and complex learning problems. In the last decade, SGD has been improved in various directions from the viewpoint of Nesterov acceleration~\citep{nesterov1983method}, variance reduction~\citep{johnson2013accelerating,schmidt2017minimizing,defazio2014saga,fang2018near} and adaptive learning rates~\citep{duchi2010adaptive,kingma2015adam,zhou2018convergenceb}.

Motivated by the increasing popularity, researchers have studied the theoretical behavior of stochastic optimization. Depending on the property of objective functions, one can measure the progress of optimization in terms of different performance metrics. For strongly convex problems, one can use the distance between the output model and the best model as the performance measure since there is only a unique minimizer~\citep{bottou2018optimization,DBLP:conf/colt/ZhangZ19}. For convex problems, one can develop convergence rates in terms of functional suboptimality gap since there may exist several models with the same global function value~\citep{zhang2004solving}. For nonconvex and smooth problems, one can measure the performance through the magnitude of gradients since an algorithm is only guaranteed to find a local minimum~\citep{ghadimi2013stochastic}. The performance metric becomes more tricky for nonconvex and nonsmooth problems~\citep{davis2019stochastic,davis2021graphical}. For an objective function $\psi$, neither the functional suboptimality gap $\psi(\bw_t)-\inf\psi(\bw)$, nor the stationarity measure, $\text{dist}(0,\partial\psi(\bw_t))$, necessarily decay to zero along the optimization process~\citep{davis2019stochastic}. Here $\bw_t$ denotes an iterate of the algorithm, $\partial\psi(\bw_t)$ denotes the subdifferential and $\text{dist}$ denotes the Euclidean distance function. Recently, \citet{davis2019stochastic} proposed to use the Moreau envelope $\psi_\lambda(\bw)=\inf_{\bv}\big\{\psi(\bv)+\frac{1}{2\lambda}\|\bv-\bw\|_2^2\big\}$ as a useful potential function to study stochastic optimization for weakly convex problems~\footnote{A function is weakly convex if eigenvalues of Hessian matrices are lower bounded by a negative value.}. An intuitive understanding is that a small gradient $\|\nabla\psi_\lambda(\bw_t)\|_2$ implies that $\bw_t$ is near some point that is nearly stationary for the problem $\min_{\bw}\psi(\bw)$, which motivates the use of the performance measure $\|\nabla\psi_\lambda(\bw_t)\|_2$ for weakly convex problems. Weakly convex problems form an importance class of nonconvex and nonsmooth problems, with instantiations in various application domains such as phase retrieval, robust principal component analysis, covariance matrix estimation and sparse dictionary learning~\citep{davis2019stochastic}.

Most of existing studies focus on the convergence behavior of stochastic optimization algorithms from the perspective of optimization, i.e., how the trained model would behave on training examples. However, in ML we are more interested in the prediction behavior from the perspective of learning~\citep{mohri2012foundations}, i.e., how these models would behave on testing examples, which is much less studied for stochastic optimization. The gap between training and testing is a central topic in statistical learning theory (SLT). There are two major approaches to study the generalization gap: a \emph{uniform convergence} approach based on the complexity analysis of hypothesis spaces~\citep{bartlett2002rademacher} and an \emph{algorithmic stability} approach based on the sensitivity analysis of algorithms~\citep{bousquet2002stability} (for simplicity we always mean algorithmic stability when mentioning stability). Uniform convergence analysis applies to nonconvex problems, which, however, often leads to a square-root dependency on the dimensionality and therefore unfavorable for high-dimensional learning problems~\citep{feldman2019high}. Stability analysis can yield dimension-free bounds, which, however, often requires strong assumptions on loss functions such as convexity or smoothness. For example, most of the algorithmic stability analysis of stochastic optimization requires a convexity and a smoothness assumption~\citep{hardt2016train,kuzborskij2018data}. The smoothness assumption is removed in the recent study~\citep{lei2020fine,bassily2020stability}. In particular, the paper~\citep{bassily2020stability} develops matching lower bounds for convex and nonsmooth problems. For nonconvex problems, one typically requires a Polyak-\L ojasiewicz (PL) condition to get nontrivial error bounds of SGD~\citep{charles2018stability}. In the general nonconvex case, the stability analysis of SGD requires very small step sizes to get meaningful stability bounds~\citep{hardt2016train,kuzborskij2018data}, for which one cannot get meaningful optimization error bounds within reasonable computations. The strong assumption restricts the application domain of stability analysis for nonconvex and nonsmooth problems, which are often encountered in practice. To our knowledge, there is no stability analysis of stochastic optimization for problems that are simultaneously nonconvex and nonsmooth without restrictive assumptions such as the PL condition.

In this paper, we initialize the stability and generalization analysis of stochastic optimization for weakly convex problems, where the objective functions are nonconvex and nonsmooth. As a warm up, we first consider convex and nonsmooth problems, then nonconvex and smooth problems, and finally move onto weakly convex problems. As indicated before, we require to use different metrics to measure the generalization performance, which also asks for different stability measures as well as a different connection between stability and generalization. Our contributions are as follows.  Comparisons between our results and existing results are given in Table \ref{tab:generalization} and Table \ref{tab:error}.

\begin{enumerate}[label=(\alph*),wide, labelwidth=!, labelindent=0pt]
  \item We introduce a stability measure called uniform stability in gradients, and establish its quantitative relationship to the generalization measured by gradients for smooth problems. In particular, we show that the gap between population and empirical gradients can be bounded by our stability measure plus $O(1/\sqrt{n})$, where $n$ is the sample size.
  \item We consider a specific class of nonconvex and nonsmooth problems called weakly convex problems, for which we measure the performance of trained models by Moreau envelopes. We develop, to our best knowledge, the first connection between argument stability and the generalization gap measured by Moreau envelopes. 
  \item We introduce the concept of sampling-determined algorithms, for which we establish stability bounds in terms of either function values, gradients or arguments.
  \item We apply our results to SGD and its adaptive variant. For nonconvex and smooth problems, we develop stability-based risk bounds without the PL condition. We also develop the first risk bounds in terms of Moreau envelops for weakly convex problems. 
\end{enumerate}


\vspace*{-0.210cm}
\section{Related Work\label{sec:related}}
\vspace*{-0.131cm}

In this section, we review the related work on generalization analysis. We will focus on two approach: the algorithmic stability approach and the uniform convergence approach.

\noindent\textbf{Algorithmic Stability}. We first review the related work on algorithmic stability. Algorithmic stability is a fundamental concept in SLT to measure the sensitivity of an algorithm up to a perturbation of the training dataset, which is closely related to learnability~\citep{shalev2010learnability,rakhlin2005stability}. There are various algorithmic stability concepts. Some stability concepts measure the sensitivity in terms of function values, e.g., uniform stability~\citep{bousquet2002stability}, hypothesis stability~\citep{bousquet2002stability,elisseeff2005stability}, Bayes stability~\citep{li2019generalization} and on-average stability~\citep{shalev2010learnability,kuzborskij2018data}, while others measure the sensitivity in terms of output models, e.g., argument stability~\citep{liu2017algorithmic} and on-average argument stability~\citep{lei2020fine}. A most widely used stability concept is the uniform stability~\citep{bousquet2002stability}, which can imply almost optimal generalization bounds with high probability~\citep{feldman2019high,bousquet2020sharper,klochkov2021stability}. Second moment bounds of generalization error for uniformly stable algorithms were developed \citep{bousquet2002stability}, and improved recently by considering a ``leave-one-out'' estimate~\citep{feldman2018generalization}. The celebrated connection between stability and generalization motivates the discussion of stability for many specific algorithms, including regularization algorithms~\citep{bousquet2002stability,attia2022uniform}, stochastic optimization algorithms~\citep{hardt2016train,chen2018stability,kuzborskij2018data,charles2018stability,mou2018generalization}, iterative hard thresholding~\citep{yuan2021stability}, structured prediction~\citep{london2016stability}, meta learning~\citep{maurer2005algorithmic} and transfer learning~\citep{kuzborskij2018data}. In particular, the influential work gives the first stability analysis of SGD applied to convex and smooth problems~\citep{hardt2016train}. The smoothness assumption is removed in the recent study~\citep{lei2020fine,bassily2020stability}, and a tight lower bound on the stability of SGD was developed~\citep{bassily2020stability}. Stability analysis can be also used to study the convergence of optimization error for multi-epoch SGD~\citep{koren2022benign}.

\noindent\textbf{Uniform Convergence}. Machine learning models may achieve good performance on the training dataset but bad generalization behavior, which motivates the generalization analysis by the uniform convergence approach to study the difference between training and testing over the whole hypothesis space. Initially, the uniform convergence was mainly studied in terms of \emph{function values}~\citep{bartlett2002rademacher}, which, however, is not appropriate to stochastic optimization with nonconvex loss functions. The underlying reason is that an algorithm can only guarantee to find a local minimizer (one cannot get convergence rate of training errors to the that of the best model). Then, the uniform convergence of function values~\citep{lei2021generalization} fail to show the convergence of testing errors to that of the best model. Instead, one has to turn to other performance measures such as the gradients of risks for smooth problems~\citep{ghadimi2013stochastic} and the gradients of Moreau envelope for weakly convex problems~\citep{davis2019stochastic}. In particular, \citet{ghadimi2013stochastic} gave the first nonasymptotical convergence rate of the gradient norm. Motivated by this observation, the uniform convergence for gradients have been recently studied~\citep{mei2018landscape,foster2018uniform,lei2021learning,davis2021graphical}. The work \citep{mei2018landscape} initialized the discussion on the uniform convergence of gradients by characterizing the complexity of function spaces with covering numbers, which was extended to the uniform convergence in terms of Rademacher complexities~\citep{foster2018uniform}. These discussions are devoted to control the uniform deviation between gradients of empirical and population risks under a smoothness condition. For nonsmooth problems, the gradients are not well defined since the functions may not be differentiable. This problem was recently addressed by considering the gradients of Moreau envelope of empirical/population risks~\citep{davis2021graphical}, which are appropriate stationary measures for weakly convex problems. Specifically, the uniform deviation of gradients for the Moreau envelope between empirical and population risks was studied based on covering numbers~\citep{davis2021graphical}.

Other than the above two approaches, there are also interesting discussions on generalization analysis by using tools in integral operators~\citep{smale2007learning,guo2017learning,mucke2019beating,pillaud2018statistical} and information theory~\citep{russo2016controlling,xu2017information,neu2021information,neu2022generalization}.

%

\vspace*{-0.210cm}
\section{Problem Setup\label{sec:problem}}
\vspace*{-0.131cm}

Let $\rho$ be a probability measure defined on a sample space $\zcal:=\xcal\times\ycal$, from which a dataset $S=\{z_1,\ldots,z_n\}$ are independently drawn. Based on $S$, we wish to build a model $h:\xcal\mapsto\ycal$ for prediction. We consider a parametric learning setting where the model is determined by a parameter $\bw$ in a parameter space $\wcal\subset\rbb^d$. The performance of a model $\bw$ on an example $z$ can be quantified by a loss function $f:\wcal\times\zcal\mapsto\rbb_+$. The training and testing behavior of $\bw$ then can be measured by the empirical risk $F_S(\bw):=\frac{1}{n}\sum_{i=1}^{n}f(\bw;z_i)$ and the population risk $F(\bw):=\ebb_Z[f(\bw;Z)]$,
where $\ebb_Z$ denotes the expectation w.r.t. $Z$. Let $\bw^*=\arg\min_{\bw\in\wcal}F(\bw)$ be the model with the minimal population risk in $\wcal$.
Let $A$ be a randomized learning algorithm and $A(S)$ be the output model when applying $A$ to the dataset $S$.
In this paper, we are interested in the quality of $A(S)$ in prediction under different performance measures.
We require necessary definitions on Lipschitz continuity, smoothness and convexity. Let  $\|\cdot\|_2$ denote the Euclidean norm and $\nabla g(\bw)$ denote a subgradient of $g$ at $\bw$. If $g$ is differentiable then $\nabla g(\bw)$ becomes the gradient of $g$ at $\bw$.
\begin{definition}
Let $g:\wcal\mapsto\rbb$. Let $L,\rho,G>0$.
\begin{enumerate}[label=(\alph*),wide, labelwidth=!, labelindent=0pt]
  \item We say $g$ is $L$-smooth if
  $
    \big\|\nabla g(\bw)-\nabla g(\bw')\big\|_2\leq L\|\bw-\bw'\|_2,\forall \bw,\bw'\in\wcal.
  $
  \item We say $g$ is convex if
  $
  g(\bw)\geq g(\bw')+\langle\bw-\bw',\nabla g(\bw')\rangle,\forall\bw,\bw'\in\wcal.
  $
  We say $g$ is $\rho$-weakly-convex if $\bw\mapsto g(\bw)+\frac{\rho}{2}\|\bw\|_2^2$ is convex, and $\rho$-strongly convex if $\bw\mapsto g(\bw)-\frac{\rho}{2}\|\bw\|_2^2$ is convex.
  \item We say $g$ is $G$-Lipschitz if
  $
  |g(\bw)-g(\bw')|\leq G\|\bw-\bw'\|_2,\forall\bw,\bw'\in\wcal.
  $
\end{enumerate}
\end{definition}

Weakly convex functions are widespread in applications with a common source being the composite function class: $g(\bw):=h(c(\bw))$, where $h:\rbb^m\mapsto\rbb$ is convex and $G$-Lipschitz and $c:\rbb^d\mapsto\rbb^m$ has $\beta$-Lipschitz continuous Jacobians~\citep{davis2019stochastic}.
Concrete examples include robust phase retrieval, covariance matrix estimation, sparse dictionary learning, robust PCA and conditional value-at-risk.
We will use error decomposition to study the generalization behavior of learning models. Depending on the property of learning tasks, we will introduce different error decompositions.

For \textbf{convex learning} problems, a learning algorithm can be guaranteed to produce a model with a small empirical error. Therefore, we quantify the behavior of a model by the associated population risk. A standard approach to studying the population risk is to decompose it into two error terms~\citep{bousquet2008tradeoffs}
\begin{equation}
\ebb_{S,A}\big[F(A(S))\big]-F(\bw^*)=\ebb_{S,A}\big[F(A(S))-F_S(A(S))\big]
+\ebb_{S,A}\big[F_S(A(S))-F_S(\bw^*)\big],\label{decomposition}
\end{equation}
where we have used $\ebb_{S,A}[F_S(\bw^*)]=F(\bw^*)$ since $\bw^*$ is independent of $A$ and $S$. We refer to the term $F(A(S))-F_S(A(S))$ in \eqref{decomposition} as the generalization error since it is related to the generalization from the training behavior to testing behavior. The second term $F_S(A(S))-F_S(\bw^*)$ is called the optimization error since it quantifies how well the algorithm minimizes the empirical risk. We will apply stability analysis to study the generalization error, and tools in optimization theory to study the optimization error.

For \textbf{nonconvex and smooth} learning problems, a learning algorithm can only be guaranteed to produce an approximate stationary point, i.e., a point $\bw$ with a small $\|\nabla F_S(\bw)\|_2$. In this case, the population risk is not a reasonable quality measure since there may be many local minimizers with different risks. As an alternative, we use the population gradient norm as the performance measure. We use the following error decomposition
\begin{equation}\label{decomposition-grad}
  \ebb_{S,A}\big[\|\nabla F(\bw)\|_2\big]\leq \ebb_{S,A}\big[\|\nabla F(\bw)-\nabla F_S(\bw)\|_2\big] + \ebb_{S,A}\big[\|\nabla F_S(\bw)\|_2\big].
\end{equation}
We call the first term $\|\nabla F(\bw)-\nabla F_S(\bw)\|_2$ the generalization error for smooth problems, and $\|\nabla F_S(\bw)\|_2$ the optimization error (empirical gradient norm). We will introduce a stability concept as well as its connection to generalization to study the generalization error for nonconvex problems. The optimization error is well studied in the literature~\citep{ghadimi2013stochastic}.

For \textbf{weakly convex} learning problems, we cannot measure the quality of a model by gradients since the function may not be differentiable. An elegant performance measure is in terms of the Moreau envelope. Intuitively, Moreau envelope of $f$ is a smoothed approximation of $f$. An illustration of the Moreau envelope  was given in Fig. 1 of  \citet{davis2019stochastic}.
\begin{definition}[Moreau envelope]
For any $\lambda>0$ and $\psi:\wcal\mapsto\rbb$, we define the Moreau envelope (with parameter $\lambda$) $\psi_\lambda:\wcal\mapsto\rbb$ by
\[
\psi_\lambda(\bw)=\min_{\bv\in\rbb^d}\big\{\psi(\bv)+1/(2\lambda)\|\bw-\bv\|_2^2\big\}
\]
and the proximal operator $\prox_{\lambda\psi}:\rbb^d\mapsto\rbb^d$ by
\[
\prox_{\lambda\psi}(\bw)=\arg\min_{\bv\in\rbb^d}\big\{\psi(\bv)+1/(2\lambda)\|\bw-\bv\|_2^2\big\}.
\]
\end{definition}
Standard results show that as long as $\psi$ is $\rho$-weakly-convex and $\lambda<1/\rho$, the envelope $\psi_\lambda$ is strongly smooth with the gradient given by $\nabla\psi_\lambda(\bw)=\lambda^{-1}\big(\bw-\prox_{\lambda\psi}(\bw)\big)$, where $\nabla\psi_\lambda(\bw)$ denotes $\nabla(\psi_\lambda)(\bw)$. For smooth $\psi$, the norm of $\nabla\psi_\lambda(\bw)$ is proportional to the magnitude of the true gradient $\nabla\psi$. For nonsmooth $\psi$, it was shown that $\|\nabla\psi_\lambda(\bw)\|_2$ has an intuitive interpretation in terms of near-stationarity of the target problem $\min_{\bw}\psi(\bw)$~\citep{davis2019stochastic}. Therefore, we use $\|\nabla F_{1/(2\rho)}(\bw)\|_2$ to quantify the generalization behavior of $\bw$ for $\rho$-weakly-convex $F$ ($F_{1/(2\rho)}$ means the Moreau envelope of $F$ with the parameter $1/(2\rho)$). We need the following error decomposition in this case
\begin{equation}\label{decomposition-envelope}
  \ebb_{S,A}\big[\|\nabla F_{1/(2\rho)}(\bw)\|_2\big]\leq \ebb_{S,A}\big[\|\nabla F_{1/(2\rho)}(\bw)- \nabla F_{S,1/(2\rho)}(\bw)\|_2\big]
   + \ebb_{S,A}\big[\|\nabla F_{S,1/(2\rho)}(\bw)\|_2\big],
\end{equation}
where we denote $F_{S,1/(2\rho)}:=(F_S)_{1/(2\rho)}$. We call the first term $\|\nabla F_{1/(2\rho)}(\bw)-\nabla F_{S,1/(2\rho)}(\bw)\|_2$ the generalization error for weakly-convex (possibly nonsmooth) problems, and $\|\nabla F_{S,1/(2\rho)}(\bw)\|_2$ the optimization error. We will introduce a novel connection between argument stability and generalization to study the generalization error for weakly-convex problems. The optimization error on $\|\nabla F_{S,1/(2\rho)}(\bw)\|_2$ is well studied in the literature~\citep{davis2019stochastic}.

We summarize our results and give comparisons with existing results in Table \ref{tab:generalization} and Table \ref{tab:error}. We consider two classes of problems: smooth \& nonconvex problems, and weakly convex \& nonsmooth problems. Table \ref{tab:generalization} considers the generalization gap, while Table \ref{tab:error} considers the error bounds for SGD.
\begin{table}[htbp]\centering
\begin{tabular}{|c|c|c|}
  \hline
  Problems & Reference & Bounds \\ \hline
   \multirow{2}{*}{smooth \& nonconvex} & Mei et al (2018)  &  $O(\sqrt{d(\log n)/n})$ \\
   & Thm.~\ref{thm:stab-gen-grad} (our work) &  $O(\epsilon+n^{-\frac{1}{2}})$ \\ \hline
   \multirow{2}{*}{\parbox{3cm}{weakly convex \& nonsmooth}} & Davis and Drusvyatskiy (2021) &  $O(\sqrt{d/n})$ \\
   & Thm.~\ref{thm:gen-argument} (our work) &  $O(\sqrt{\epsilon}+n^{-\frac{1}{2}})$ \\
  \hline
\end{tabular}
\caption{Generalization bounds. For smooth and nonconvex problems, the generalization bounds are derived for $\|\nabla F(A(S))-\nabla F_S(A(S))\|_2$. For weakly convex and nonsmooth problems, the generalization bounds are derived for $\big\|\nabla F_{S,1/(2\rho)}(A(S))-\nabla F_{1/(2\rho)}(A(S))\big\|_2$. The existing generalization bounds are based on uniform convergence approach, and admit a square-root dependency on the dimension. Our generalization bounds depend on the stability parameter $\epsilon$.\label{tab:generalization}}
\end{table}

\begin{table}[htbp]
\begin{tabular}{|c|c|c|}
  \hline
  Problems & Reference & Bounds \\ \hline
  \multirow{2}{*}{smooth \& nonconvex} & Ghadimi and Lan (2013) & $\|\nabla F_S(\mathbf{w}_r))\|=O(T^{-\frac{1}{4}})$ \\
   & Prop.~\ref{prop:sgd-nonconvex} (our work) & $\|\nabla F(\mathbf{w}_r)\|_2=O(n^{-\frac{1}{6}})$ \\ \hline
  \multirow{2}{*}{\parbox{3cm}{weakly convex \& nonsmooth}} &  Davis and Drusvyatskiy (2019) & $\|\nabla F_{S,1/(2\rho)}(\mathbf{w}_r)\|_2=O(T^{-\frac{1}{4}})$ \\
   & Prop.~\ref{prop:sgd-wc} (our work) & $\|\nabla F_{1/(2\rho)}(\mathbf{w}_r)\|_2=O(n^{-\frac{1}{6}})$  \\
  \hline
\end{tabular}
\caption{Error bounds for SGD. The existing analysis considers the performance of SGD on the empirical risk $F_S$, while our results consider the performance of SGD on the population risk $F$. Here $T$ is the number of iterations and $\mathbf{w}_r$ is a randomly selected SGD iterate.\label{tab:error}}
\end{table}

\vspace*{-0.210cm}
\section{Stability and Generalization\label{sec:stab-gen}}
\vspace*{-0.131cm}

\subsection{Connecting Stability and Generalization}
\vspace*{-0.066cm}
Algorithmic stability measures the insensitiveness on an algorithm under a perturbation of a training dataset by a single example.
The uniform stability and uniform argument stability were discussed in the literature~\citep{bousquet2002stability}.
To tackle the performance measure in terms of gradient norms for nonconvex learning problems, we introduce a \emph{uniform stability in gradients}.
We say $S,S'$ are neighboring datasets if they differ by at most a single example.
\begin{definition}[Uniform Stability\label{def:unif-stab}]
  Let $A$ be a randomized algorithm.
  We say $A$ is $\epsilon$-uniformly-stable in \emph{function values} if for all neighboring datasets $S,S'$, we have
  \begin{equation}\label{unif-stab}
  \sup_z\ebb_A\big[f(A(S);z)-f(A(S');z)\big]\leq\epsilon.
  \end{equation}
  We say $A$ is $\epsilon$-uniformly-argument-stable if for all neighboring datasets $S,S'$, we have
  \begin{equation}\label{argument-stab}
  \ebb_A\big[\|A(S)-A(S')\|_2\big]\leq\epsilon.
  \end{equation}
  We say $A$ is $\epsilon$-uniformly-stable \emph{in gradients} if for all neighboring datasets $S,S'$, we have
  \begin{equation}\label{grad-stab}
  \sup_z\ebb_A\big[\|\nabla f(A(S);z)-\nabla f(A(S');z)\|_2^2\big]\leq\epsilon^2.
  \end{equation}
\end{definition}

\begin{remark}\label{rem:grad-stab}\normalfont
The motivation of introducing the gradient-based stability is to use it to study the generalization performance for nonconvex problems. For nonconvex problems, an optimization algorithm generally only finds a local minimizer, and therefore one cannot use the function value to measure the convergence (the local minimizer the algorithm finds may be far away from the global minimizer and therefore the convergence in function values do not make much sense). In this case, one often studies the convergence of $\nabla F_S$ in the optimization community~\citep{ghadimi2013stochastic}.
To use this convergence to study the behavior of $A(S)$ in prediction, we need to address $\|\nabla F(A(S))-\nabla F_S(A(S))\|_2$, which, as we will see, can be achieved by stability in gradients. In summary, the stability on gradients allows us to incorporate the existing optimization error bounds to study the prediction performance as measured by $\|\nabla F(A(S))\|_2$.
\end{remark}

The connection between uniform stability in function values and generalization is given in the following lemma~\citep{shalev2010learnability,hardt2016train}. 
\begin{lemma}[Generalization via Stability in Function Values\label{lem:gen-stab}]
  Let $A$ be $\epsilon$-uniformly stable in function values. Then
  $
  \big|\ebb_{S,A}\big[F_S(A(S))-F(A(S))\big]\big|\leq\epsilon.
  $
\end{lemma}

Our first result is a connection between generalization and stability in gradients. This result cannot be derived by using the standard arguments in the literature \citep{shalev2010learnability,hardt2016train} since one can not exchange the summation operator and norm. We will give more explanations in the proof, which is given in Section \ref{sec:proof-stab-gen-grad}.
\begin{theorem}[Generalization via Stability in Gradients]\label{thm:stab-gen-grad}
  Let $A$ be $\epsilon$-uniformly-stable in gradients. Assume for any $z$, the function $f(\bw;z)$ is differentiable. Then
  \begin{equation}\label{stab-gen-grad-a}
    \ebb_{S,A}\big[\|\nabla F(A(S))-\nabla F_S(A(S))\|_2\big] \leq 4\epsilon
    +\sqrt{n^{-1}\ebb_S\Big[\mathbb{V}_Z(\nabla f(A(S);Z))\Big]},
  \end{equation}
  where
  $
  \mathbb{V}_Z(\nabla f(A(S);Z))=\ebb_Z\big[\|\nabla f(A(S);Z)-\ebb_Z[\nabla f(A(S);Z)]\|_2^2\big]
  $
  is the variance of $\nabla f(A(S);Z)$ as a function of the random variable $Z$.
\end{theorem}
\begin{remark}\normalfont
  Note the left-hand side of Eq. \eqref{stab-gen-grad-a} can be addressed by the uniform convergence
  of gradients $\sup_{\bw\in\wcal}\|\nabla F(\bw)-\nabla F_S(\bw)\|_2$, which was established in terms of covering numbers~\citep{mei2018landscape} and Rademacher complexities~\citep{foster2018uniform}. These bounds generally involve a square-root dependency on the dimension of $\wcal$. As a comparison, Theorem \ref{thm:stab-gen-grad} considers the convergence of empirical gradients to population gradients at the output model $A(S)$. Therefore, it implies dimension-free bounds which would be effective for high-dimensional learning problems.
\end{remark}

Our second result is a connection between the uniform argument-stability and generalization measured by the Moreau envelope for weakly convex problems. Theorem \ref{thm:gen-argument} shows that the difference between empirical and population gradients of the Moreau envelope at $A(S)$ can be bounded by the uniform argument stability of $A$. With this result, we can transfer the existing bound on $\|\nabla F_{S,1/(2\rho)}\|_2$ to $\|\nabla F_{1/(2\rho)}\|_2$ on the performance of models for prediction. The proof of Theorem \ref{thm:gen-argument} is totally different from that of Theorem \ref{thm:stab-gen-grad}. The proof is given in Section \ref{sec:proof-gen-argument}. 
\begin{theorem}[Generalization via Uniform Argument Stability\label{thm:gen-argument}]
Let $A$ be $\epsilon$-argument stable.
Assume for any $z$, the function $f(\bw;z)$ is $G$-Lipschitz continuous.
Assume for any $S$, the function $F_S$ is $\rho$-weakly-convex and $F$ is $\rho$-weakly-convex.
Then
\begin{equation}\label{gen-argument}
\ebb\big[\big\|\nabla F_{S,1/(2\rho)}(A(S))-\nabla F_{1/(2\rho)}(A(S))\big\|_2\big]\leq
\frac{4G}{\sqrt{n}}+\sqrt{32G\epsilon\rho}.
\end{equation}
\end{theorem}
\begin{remark}\normalfont
  For $\rho$-weakly convex $f$, the uniform convergence $\sup_{\bw\in\wcal}\|\nabla F_{S,1/(2\rho)}(\bw)-\nabla F_{1/(2\rho)}(\bw)\|_2$ was studied in terms of the covering number of $\wcal$~\citep{davis2021graphical}, which generally involves a square-root dependency on the dimensionality. For example, if $\wcal$ is a ball in $\rbb^d$, then the following result was established~\citep{davis2021graphical}
  \begin{equation}\label{unif-davis}
  \sup_{\bw\in\wcal}\|\nabla F_{S,1/(2\rho)}(\bw)-\nabla F_{1/(2\rho)}(\bw)\|_2=O\big(G\sqrt{d/n}\big).
  \end{equation}
  The underlying reason to consider a uniform convergence is noting the dependency of $A(S)$ in Eq. \eqref{gen-argument} on $S$. We address this dependency by giving a bound in terms of the argument stability of $A$.
  Theorem \ref{thm:gen-argument} yields dimension-free bounds since it only considers the convergence of $\nabla F_{S,1/(2\rho)}$ to $\nabla F_{1/(2\rho)}$ at the particular output model $A(S)$. 
\end{remark}

Finally, we give a high-probability bound on the generalization gap measured by the Moreau envelope. The proof is given in Section \ref{sec:proof-gen-argument}. 
\begin{theorem}[High-probability Bound via Uniform Argument Stability\label{thm:gen-argument-hp}]
Let $A$ be $\epsilon$-argument stable almost surely, i.e., $\|A(S)-A(S')\|_2\leq \epsilon$ for any neighboring $S,S'$.
Assume for any $z$, the function $f(\bw;z)$ is $G$-Lipschitz continuous and $f(A(S);z)=O(1)$ almost surely.
Assume for any $S$, the function $F_S$ is $\rho$-weakly-convex and $F$ is $\rho$-weakly-convex.
For any $\delta\in(0,1)$, the following inequality holds with probability at least $1-\delta$
\begin{multline*}
\!\!\!\!\!\!\big\|\nabla F_{S,1/(2\rho)}(A(S))\!-\!\nabla F_{1/(2\rho)}(A(S))\big\|_2\!=\!
O\Big(\big(Gn^{-\frac{1}{2}}\!+\!\sqrt{G\epsilon\rho}\big)\sqrt{\log(n)\log(1/\delta)}\!+\!\big(n^{-1}\rho^{2}\log(1/\delta)\big)^{\frac{1}{4}}\Big).
\end{multline*}
\end{theorem}

\vspace*{-0.066cm}
\subsection{Stability Bounds}
\vspace*{-0.066cm}

We now consider a class of randomized algorithms called sampling-determined algorithms for our stability analysis. We say a randomized algorithm $A$ is symmetric if its output is independent on the order of the elements in the training set.
\begin{definition}[Sampling-determined Algorithm]
  Let $A$ be a randomized algorithm which randomly chooses an index sequence $I(A)=\{i_t\}$ to build stochastic gradients.
  We say a symmetric algorithm $A$ is sampling-determined if the output model is determined by $\{z_j:j\in I(A)\}$. To be precise, $A(S)$ is independent of $z_j$ if $j\not\in I$. 
\end{definition}

An important property of sampling-determined algorithms is that these algorithms will produce the same model when applied to two neighboring datasets if the differing example is not selected in the algorithm.
For example, if two neighboring datasets differ by the first example and the index $1$ is not selected by the algorithm, then the algorithm would produce the same model when applied to these two neighboring datasets.
This property is critical for us to study the stability. The class of sampling-determined algorithms include several famous randomized algorithms. Below, we give some representative algorithms. The first algorithm is the SGD, which is a most simple and most popular stochastic optimization algorithm.
Let $\Pi_{\wcal}(\bw)$ denote the projection of $\bw$ onto $\wcal$. Note $\wcal$ can be $\rbb^d$ and in this case there is no projection.
\begin{definition}[Stochastic Gradient Descent]
Let $\bw_1=0\in\rbb^d$ be an initial point and $\{\eta_t\}_t$ be a sequence of positive step sizes. SGD updates models by
$
  \bw_{t+1}=\Pi_{\wcal}\big(\bw_t-\eta_t\nabla f(\bw_t;z_{i_t})\big),
$
where $\nabla f(\bw_t,z_{i_t})$ denotes a subgradient of $f$ w.r.t. the first argument and $i_t$ is independently drawn from the uniform distribution over $[n]:=\{1,2,\ldots,n\}$.
\end{definition}

The second algorithm is an adaptive variant of SGD, which introduces a sequence $\{b_t^2\}$ to store the accumulated gradient norm square~\citep{duchi2010adaptive,li2019convergence,ward2020adagrad}. We then set the step size as the reciprocal of $b_t$ multiplied by a parameter $\eta$~\citep{ward2020adagrad}. This algorithm has a nice advantage of being able to adapt the level of stochastic noise of the problem, and can achieve robust convergence without the need to fine-tune stepsize schedule.
\begin{definition}[AdaGrad-Norm] 
Let $\bw_1=0\in\rbb^d$, $b_0>0$ and $\eta>0$. At each iteration, we first draw $i_t$ from the uniform distribution over $[n]$ and update $\{b_t\},\{\bw_t\}$ by
\begin{equation}\label{adagrad}
     b_{t}^2 = b_{t-1}^2 + \|\nabla f(\bw_t;z_{i_t})\|_2^2,\qquad
     \bw_{t+1}  = \Pi_{\wcal}\Big(\bw_t - \frac{\eta}{b_{t}} \nabla f(\bw_t;z_{i_t})\Big).
\end{equation}
\end{definition}
\begin{remark}\normalfont
  Let $A$ be either SGD or AdaGrad-Norm  with $T$ iterations. Note $A(S)$ does not depend on $z_j$ if $j\in[n]$ is not selected in the implementation of $A$. Therefore, both SGD and AdaGrad-Norm are sampling-determined algorithms and $I(A)=\{i_1,\ldots,i_T\}$. It is also clear from the definition that Adam is a sampling-determined algorithm.
\end{remark}

\begin{remark}\label{rem:svrg}\normalfont
There are also some randomized algorithms that are not sampling-determined. A notable example is the stochastic variance reduction gradient (SVRG)~\citep{johnson2013accelerating}. Note that SVRG is implemented in epochs, for each of which we need to compute the full gradient at a reference point. Therefore, SVRG will produce different models when applied to neighboring datasets even if the differing example is not selected to compute a stochastic gradient. One can also check that other variance reduction algorithms are not sampling-determined, including  stochastic average gradient~\citep{schmidt2017minimizing} and SAGA~\citep{defazio2014saga}.
\end{remark}

The following theorem to be proved in Section \ref{sec:proof-stab-gen} (supplementary material) establishes the uniform stability bounds for sampling-determined algorithms. It shows that the uniform stability of a sampling-determined algorithm $A$ can be bounded by the probability of an index not selected in $I(A)$. We consider stability in function values (Part (a)), stability in gradients (Part (b)) and stability in arguments (Part (c)). The proof is motivated by the arguments in \citet{hardt2016train}.
\begin{theorem}\label{thm:sampling-determined}
  Let $A$ be a sampling-determined algorithm and $S,S'$ be neighboring datasets.
  \begin{enumerate}[label=(\alph*)]
    \item If $\sup_z\ebb_A[f(A(S);z)|n\in I(A)]\leq B$ for any $S$, then \[\!\!\!\!\sup_z\ebb_A\big[f(A(S);z)-f(A(S');z)\big]\leq 2B \cdot\mathrm{Pr}\{n\in I(A)\}.\]
    \item If $\sup_z\ebb_A[\|\nabla f(A(S);z)\|_2^2|n\in I(A)]\leq G^2$ for any $S$, then
    \[
      \sup_z\ebb_A\big[\|\nabla f(A(S);z)-\nabla f(A(S');z)\|_2^2\big]
      \leq 4G^2\cdot\mathrm{Pr}\{n\in I(A)\}.
    \]
    \item If $\ebb_A[\|A(S)\|_2|n\in I(A)]\leq R$ for any $S$, then $\ebb_A\big[\|A(S)-A(S')\|_2\big]\leq 2R\cdot\mathrm{Pr}\{n\in I(A)\}.$
  \end{enumerate}
\end{theorem}

We derive Corollary~\ref{cor:gen} by computing $\mathrm{Pr}\{n\in I(A)\}$. The proof is given in~Section~\ref{sec:proof-stab-gen}. 
\begin{corollary}\label{cor:gen}
  Let $A$ be SGD or AdaGrad-Norm with $T$ iterations and $S,S'$ be neighboring datasets.
  \begin{enumerate}[label=(\alph*)]
    \item If $\sup_z\ebb_A[f(A(S);z)|n\in I(A)]\leq B,\forall S$, then $\sup_z\ebb_A\big[f(A(S);z)\!-\!f(A(S');z)\big]\leq \frac{2BT}{n}.$
    \item If $\sup_z\ebb_A[\|\nabla f(A(S);z)\|_2^2|n\in I(A)]\leq G^2$ for any $S$, then $\sup_z\ebb_A\big[\|\nabla f(A(S);z)-\nabla f(A(S');z)\|_2^2\big]\leq \frac{4G^2T}{n}.$
    \item If $\ebb_A[\|A(S)\|_2|n\in I(A)]\leq R$ for any $S$, then $\ebb_A\big[\|A(S)-A(S')\|_2\big]\leq \frac{2RT}{n}.$
  \end{enumerate}
\end{corollary}
\begin{remark}\label{rem:symmetric}\normalfont
  Since we consider symmetric algorithms, the condition $n\in I(A)$ can be replaced by $i\in I(A)$ for any $i\in[n]$. Both Theorem \ref{thm:sampling-determined} and Corollary \ref{cor:gen} require boundedness assumptions on either function values, gradients and arguments, which hold immediately if we impose a projection operator on $A(S)$. Note we do not require a projection for each iterate. A projection for the final output $A(S)$ suffices for our analysis.
\end{remark}

\vspace*{-0.210cm}
\section{Applications to Stochastic Gradient Descent\label{sec:sgd}}
\vspace*{-0.131cm}

We now apply our stability results to SGD. We denote $B\asymp \widetilde{B}$ if there exist constants $c_1,c_2>0$ such that $c_1\widetilde{B}<B\leq c_2\widetilde{B}$. Recall $n$ is the sample size and $T$ is the iteration number.
We will consider different problem settings: convex and smooth cases, nonconvex and smooth cases, and weakly convex cases. All the proofs in this subsection can be found in Section \ref{sec:proof-sgd}. We will give applications to adaptive gradient descent in Section \ref{sec:adagrad}, and differentially private SGD in Section \ref{sec:dp-sgd}.

\noindent\textbf{Convex and Nonsmooth Problems}.
In Proposition~\ref{prop:sgd}, we show SGD applied to convex and nonsmooth problems can imply the excess population risk bounds $O(n^{-\frac{1}{3}})$ with $O(n^{\frac{2}{3}})$ iterations. The algorithm is computationally efficient in the sense that SGD with $T\asymp n^{\frac{2}{3}}$ iterations can at most imply optimization error bounds $O(1/\sqrt{T})=O(n^{-\frac{1}{3}})$. Therefore, our analysis implies excess risk bounds of the same order of optimization error bounds with the same computation complexity. There is no additional cost by going from optimization to generalization if we run $O(n^{\frac{2}{3}})$ iterations. This proposition is not a main result since our focus is on nonconvex case. We present it just as a byproduct. Recall $\bw^*$ is a minimizer of the population risk $F$ and we assume $\|\bw^*\|_2$ is finite.

\begin{proposition}[Convex and Nonsmooth Case\label{prop:sgd}]
Let $\{\bw_t\}_t$ be the sequence produced by SGD and $\ebb[\|\nabla f(\bw_t;z_{i_t})\|_2^2]\leq G^2$ for all $t\in[T]$. Let $A$ output $\bar{\bw}_T=\frac{1}{T}\sum_{t=1}^{T}\bw_t$. If $F_S$ is convex, $\eta_t=\eta$ and $\sup_z\ebb_A[f(A(S);z)|n\in I(A)]\leq B$, then
  \begin{equation}\label{sgd-a}
  \ebb_{S,A}[F(\bar{\bw}_T)]-F(\bw^*)=O\Big(\frac{T\eta^2G^2+\|\bw^*\|_2^2}{T\eta}\Big)+O\big(BT/n\big).
  \end{equation}
If
$\eta\!\asymp\! n^{-\frac{1}{3}}\|\bw^*\|_2/G,T\!\asymp\! n^{\frac{2}{3}}G\|\bw^*\|_2/B$ we have
$
\ebb[F(\bar{\bw}_T)]\!-\!F(\bw^*)\!=\!O(\big(G\|\bw^*\|_2+B\big)n^{-\frac{1}{3}}).
$
\end{proposition}
\begin{remark}\label{rem:sgd-a}\normalfont
  We compare Proposition \ref{prop:sgd} with existing results. The following excess risk bounds of SGD without smoothness assumptions were established~\citep{lei2020fine,bassily2020stability}
  \begin{equation}\label{lei-fine}
  \ebb_{S,A}[F(\bar{\bw}_T)]-F(\bw^*)=O\big(G^2\sqrt{T}\eta+T\eta G^2/n+\|\bw^*\|_2^2/(T\eta)\big).
  \end{equation}
  By setting $T\asymp n^2$ and $\eta\asymp T^{-\frac{3}{4}}\|\bw^*\|_2/G$, the above bound implies the excess risk bounds $\ebb_{S,A}[F(\bar{\bw}_T)]-F(\bw^*)=O(G\|\bw^*\|_2n^{-\frac{1}{2}})$. As a comparison, our analysis implies the bounds $O((G\|\bw^*\|_2+B)n^{-\frac{1}{3}})$. However, the bound \eqref{lei-fine} requires $O(n^2)$ iterations to achieve this optimal risk bounds, which is computationally expensive. As a comparison, our analysis requires $O(n^{\frac{2}{3}}G\|\bw^*\|_2/B)$ iterations to achieve the bound $O((G\|\bw^*\|_2+B)n^{-\frac{1}{3}})$. To achieve the bound $O(G\|\bw^*\|_2n^{-\frac{1}{3}})$, the existing analysis \citep{lei2020fine,bassily2020stability} requires to run SGD with $O(n^{\frac{4}{3}})$ iterations. Indeed, the right-hand-side of \eqref{lei-fine} is at least of the order of
  $
  O\big(G^2\sqrt{T}\eta+\|\bw^*\|_2^2/(T\eta)\big)\geq O(G\|\bw^*\|_2T^{-\frac{1}{4}}).
  $
  Setting $T^{-\frac{1}{4}}=n^{-\frac{1}{3}}$ gives the complexity requirement $T=n^{\frac{4}{3}}$, which is larger than the iteration complexity $n^{\frac{2}{3}}G\|\bw^*\|_2/B$ in Proposition  \ref{prop:sgd}.  Note we require an assumption $\sup_z\ebb_A[f(A(S);z)|n\in I(A)]\leq B$ in Proposition \ref{prop:sgd}, which is not required in \citet{lei2020fine,bassily2020stability}. The discussion in \citet{bassily2020stability} requires a Lipschitz assumption and imply high-probability bounds, while we require the assumption $\ebb[\|\nabla f(\bw_t;z_{i_t})\|_2^2]\leq G^2$ and derive bounds in expectation.  Furthermore, a tight lower bound on the stability is developed in \citet{bassily2020stability}. 

  Excess risk bounds of the order $O(G\|\bw^*\|_2n^{-\frac{1}{3}}\log n)$ were also established for SGD based on the uniform convergence approach~\citep{lin2016generalization}. Their discussions consider kernel methods and would imply dimension-dependent bounds if applied to general nonlinear models. As a comparison, our stability analysis always yields dimension-free bounds.
\end{remark}

\noindent\textbf{Nonconvex and Smooth Problems}.
We now consider the performance of SGD for nonconvex and smooth problems. In the remainder, we always let $r$ be randomly selected from the uniform distribution over $[T]$. We show SGD with $O(n^{\frac{2}{3}})$ iterations achieves the population gradient bound $\ebb_{S,A,r}\big[\|\nabla F(\bw_r)\|_2\big]=O(n^{-\frac{1}{6}})$. Again, this result shows considering generalization does not bring additional computation cost since SGD with $T\asymp n^{\frac{2}{3}}$ iterations is only guaranteed to achieve empirical gradient bounds $\ebb_{S,A,r}\big[\|\nabla F_S(\bw_r)\|_2\big]=O(n^{-\frac{1}{6}})$~\citep{ghadimi2013stochastic}. That is, with $n^{\frac{2}{3}}$ iterations, our population gradient bounds match the existing empirical gradient bounds.

\begin{proposition}[Nonconvex and Smooth Case]\label{prop:sgd-nonconvex}
Let $\{\bw_t\}_t$ be produced by SGD with $\eta_t=\eta$ and $\ebb[\|\nabla f(\bw_t;z_{i_t})\|_2^2]\leq G^2$ for all $t\in[T]$.
If $A(S)=\bw_r$, $F_S$ is $L$-smooth and \[\sup_z\ebb_A[\|\nabla f(A(S);z)\|_2^2|n\in I(A)]\leq G^2,\] then
\[
\ebb_{S,A,r}\big[\|\nabla F(\bw_r)\|_2\big] = O\Big(\frac{G\sqrt{T}}{\sqrt{n}} +  \frac{G\sqrt{T}\eta+1}{\sqrt{T\eta}}\Big).
\]
If $T\asymp n^{\frac{2}{3}}/G^{\frac{2}{3}},\eta\asymp 1/(G\sqrt{T})$, we get $\ebb\big[\|\nabla F(\bw_r)\|_2\big]=O(G^{\frac{2}{3}}n^{-\frac{1}{6}})$.
\end{proposition}

\begin{remark}\label{rem:sgd-smooth}\normalfont
  We compare our bounds with existing results. For nonconvex, smooth and Lipschitz loss functions, the uniform stability bound of order $O(n^{-1}T^{\frac{Lc}{Lc+1}})$ was established for SGD with $\eta_t\leq c/t$~\citep{hardt2016train}. While this analysis gives nontrivial bounds on the generalization gap, the proposed step size is small to enjoy a good decay of optimization errors. Indeed, with this step size one can only derive optimization error bounds $\ebb_{S,A,r}\big[\|\nabla F_S(\bw_r)\|_2\big]=O(1/\log T)$. One cannot trade-off the generalization bounds $O(n^{-1}T^{\frac{Lc}{Lc+1}})$ and optimization error bounds $O(1/\log T)$ for a non-vacuous population gradient bound. Indeed, to get a non-vacuous bound, one requires $T=O(n^{\frac{Lc+1}{Lc}})$. However, in this case the optimization error bounds become $O(1/\log n)$, which are very slow. As a comparison, our discussion suggests a step size $\eta_t\asymp n^{-\frac{1}{3}}$ for a significantly better population risk bound $O(n^{-\frac{1}{6}})$. We should mention that the discussion in \citet{hardt2016train} considers the stability in function values, while we consider stability in gradients. High probability bounds on a weighted average of $\|\nabla F(\bw_t)\|_2^2$ were developed in \citet{lei2021learning}. Their discussions use a uniform convergence approach and therefore admits a square-root dependency on the dimensionality. As a comparison, Proposition \ref{prop:sgd-nonconvex} yields dimension-free bounds.
\end{remark}

We can improve population gradient bounds under a strong growth condition (SGC),
which connects the rates at which the stochastic gradients shrink to the full gradient~\citep{vaswani2019fast}.
\begin{definition}
We say SGC holds if
$
\frac{1}{n}\sum_{i=1}^{n}[\|\nabla f(\bw;z_i)\|_2^2] \leq \rho\|\nabla F_S(\bw)\|_2^2.
$
\end{definition}

Proposition \ref{prop:sgd-sgc} shows that the learning performance improves under the SGC condition.
\begin{proposition}[Nonconvex, Smooth and SGC Case\label{prop:sgd-sgc}]
Assume for all $z$, the function $\bw\mapsto f(\bw;z)$ is $L$-smooth and SGC holds. Let $\{\bw_t\}_t$ be produced by SGD with $\eta_t=1/(\rho L)$ and suppose $\ebb_{S,A}[\|\nabla f(\bw_t;z_{i_t})\|_2^2]\leq G^2$ for all $t\in[T]$.
If $A(S)=\bw_r$, $T\asymp \sqrt{L\rho n}/G$ and \[\sup_z\ebb_A[\|\nabla f(A(S);z)\|_2^2|n\in I(A)]\leq G^2,\] then $\ebb_{S,A}\big[\|\nabla F(\bw_r)\|_2\big]=O((L\rho G^2/n)^{\frac{1}{4}})$.
\end{proposition}

\noindent\textbf{Weakly Convex Problems}.
Finally, we consider weakly convex problems. Note we impose a bounded subgradient assumption $\ebb[\|\nabla f(\bw_t;z_{i_t})\|_2^2]\leq G^2$ as in \citet{davis2019stochastic}. In the appendix \ref{sec:proof-bound-var}, we will relax this assumption as $\ebb[\|\nabla f(\bw_t;z_{i_t})\|_2^2]\leq B_1\ebb[f(\bw_t;z_{i_t})]+B_2$ for some $B_1,B_2>0$ and derive the corresponding convergence rates of SGD. To our knowledge, this convergence analysis under the relaxed condition is new for SGD with weakly convex problems.

\begin{proposition}[Weakly-convex Case]\label{prop:sgd-wc}
Let $\{\bw_t\}_t$ be given by SGD with $\eta_t=\eta$ and $A(S)=\bw_r$. Assume $\ebb_{S,A}[\|\nabla f(\bw_t;z_{i_t})\|_2^2]\leq G^2, \ebb_A[\|A(S)\|_2|n\in I(A)]\leq R$.
If  $F_S$ is $\rho$-weakly convex, then
\[
\ebb_{S,A,r}\big[\|\nabla F_{1/(2\rho)}(\bw_r)\|_2\big] = O\Big(G\sqrt{\rho\eta}+\sqrt{GR\rho T/n}+1/\sqrt{T\eta}\Big).
\]
If $T\asymp n^{\frac{2}{3}}/(R^{\frac{2}{3}}\rho^{\frac{1}{3}})$ and $\eta\asymp 1/(G\sqrt{\rho T})$, we get $\ebb\big[\|\nabla F_{1/(2\rho)}(\bw_r)\|_2\big]=O(\sqrt{G}\rho^{\frac{1}{3}}R^{\frac{1}{6}}/n^{\frac{1}{6}})$.
\end{proposition}
\begin{remark}\normalfont
  For weakly convex problems, the convergence rate \[\ebb\big[\|\nabla F_{S,1/(2\rho)}(\bw_r)\|_2\big]=O(G^{\frac{1}{2}}\rho^{\frac{1}{4}}T^{-\frac{1}{4}})\] was established for SGD with $T$ iterations~\citep{davis2019stochastic}. This result is impressive since neither the Moreau envelope nor the proximal map of $F_S$ explicitly appear in the implementation of SGD. This result shows the behavior of SGD on training examples, which we extend to the generalization behavior of SGD on testing examples. Note our analysis requires to set $T\asymp n^{\frac{2}{3}}$ and therefore can only imply the bound of the order $O((G\rho)^{\frac{1}{3}}n^{-\frac{1}{6}})$. It would be interesting to further improve the risk bound here.

  Population risk bounds of gradient descent were recently studied for weakly convex problems~\citep{richards2021learning,richards2021stability}. Their discussions require the weak convexity parameter to be sufficiently small for meaningful generalization. As a comparison, our discussion does not require this assumption. Furthermore, their discussions consider smooth problems with Lipshictz continuous Hessians and focus on gradient descent~\citep{richards2021learning}, while our discussions apply to SGD with nonsmooth problems.
\end{remark}

\begin{remark}
  A drawback of our stability analysis is that it requires $T=o(n)$ to get non-vacuous stability bounds, and therefore can only imply suboptimal generalization bounds. Better generalization bounds can be obtained for one-pass SGD by applying the standard online-to-batch technique~\citep{hazan2016introduction,cesa2004generalization} to the existing optimization error bounds~\citep{ghadimi2013stochastic,davis2019stochastic}. For the one-pass SGD, each training example is used only once and then there is no necessity to consider the generalization issues. Our algorithm differs from the one-pass SGD since it is possible that a single training example is selected several times even if $T=o(n)$, and then there is still a need to consider the generalization issues. How to improve the stability analysis for nonconvex SGD in the case $T>n$ remains an interesting open question.
\end{remark}

\vspace*{-0.210cm}
\section{Conclusions\label{sec:conclusion}}
\vspace*{-0.131cm}

We provide a systematic study on the stability and generalization analysis of stochastic optimization for problems that can be either nonconvex or nonsmooth. We consider three stability measures: the stability by function values, the stability by gradients and the stability by arguments, which are used to study convex and nonsmooth problems, nonconvex and smooth problems, and weakly convex problems, respectively. We develop connection between  stability and generalization gap measured by gradients for either the population risks or the Moreau envelopes. We then develop bounds for these stability measures for a class of sampling-determined algorithms. As a combination of these stability bounds and the connection between stability and generalization, we develop error bounds for SGD and AdaGrad-Norm, with the performance measured by either functional suboptimality, stationarity by gradients or stationarity by Moreau envelopes. It is interesting to derive sharper generalization bounds for nonconvex learning via an algorithmic stability approach. It is also very interesting to develop lower bounds for learning with weakly convex problems.

\section*{Acknowledgments}

We thank Prof. Yiming Ying for interesting discussions. We are grateful to the anonymous reviewers and the area chair for their constructive comments and suggestions.

\setlength{\bibsep}{0.22cm}

\appendix

\numberwithin{equation}{section}
\numberwithin{theorem}{section}
\numberwithin{figure}{section}
\numberwithin{table}{section}
\renewcommand{\thesection}{{\Alph{section}}}
\renewcommand{\thesubsection}{\Alph{section}.\arabic{subsection}}
\renewcommand{\thesubsubsection}{\Roman{section}.\arabic{subsection}.\arabic{subsubsection}}
\setcounter{secnumdepth}{-1}
\setcounter{secnumdepth}{3}

\mbox{}
\medskip
\normalsize


\section{Proofs on Stability and Generalization\label{sec:proof}}
\subsection{Proof of Theorem \ref{thm:stab-gen-grad}\label{sec:proof-stab-gen-grad}}
In this section, we prove the connection between generalization and uniform stability measured by gradients. For brevity, we use $\ebb[\cdot]$ to denote $\ebb_{S,A}[\cdot]$. Before proving Theorem \ref{thm:stab-gen-grad}, we first present the proof of Lemma \ref{lem:gen-stab}. This result is known in the literature~\citep{shalev2010learnability,hardt2016train,kuzborskij2018data}. We give the proof for completeness and for showing that these arguments cannot be used to prove Theorem \ref{thm:stab-gen-grad}.

\begin{proof}[Proof of Lemma \ref{lem:gen-stab}]
Let $S'=\{z'_1,\ldots,z'_n\}$ be drawn independently from $\rho$. For any $i\in[n]$, define $S^{(i)}=\{z_1,\ldots,z_{i-1},z'_i,z_{i+1},\ldots,z_n\}$. According to the symmetry between $z_i$ and $z_i'$ we have
\begin{align*}
  \ebb[F_S(A(S))-F(A(S))] &= \ebb\big[F_S(A(S))-\frac{1}{n}\sum_{i=1}^{n}F(A(S^{(i)}))\big]\\
   & = \frac{1}{n}\sum_{i=1}^{n}\ebb\big[f(A(S);z_i)-f(A(S^{(i)});z_i)\big],
\end{align*}
where the last identity holds since $A(S^{(i)})$ is independent of $z_i$.
It then follows that
\[
\big|\ebb[F_S(A(S))-F(A(S))]\big|\leq
\frac{1}{n}\sum_{i=1}^{n}\ebb\big[|f(A(S);z_i)-f(A(S^{(i)});z_i)|\big]\leq\epsilon.
\]
The proof is completed.
\end{proof}

An essential argument in proving Lemma \ref{lem:gen-stab} is to use the identity \[\ebb_{S,A}[F(A(S))]=\frac{1}{n}\sum_{i=1}^{n}\ebb_{S,A}[f(A(S^{(i)});z_i)].\] However, if we consider gradients of population risks we can only get
\[\ebb_{S,A}\big[\|\nabla F(A(S))\|_2\big]\!=\!\frac{1}{n}\sum_{i=1}^{n}\ebb_{S,A}\big[\|\ebb_{z_i}[\nabla f(A(S^{(i)});z_i)]\|_2\big],\]
where the summation is outside of $\|\cdot\|_2$. As a comparison, if we consider gradients of empirical risks we get $\|\nabla F_S(A(S))\|_2=\big\|\frac{1}{n}\sum_{i=1}^{n}\nabla f(A(S);z_i)\big\|_2$, where the summation is inside the norm. Since we cannot exchange the norm and the summation, we cannot use the argument in the proof of Lemma \ref{lem:gen-stab} to prove Theorem \ref{thm:stab-gen-grad}.

\noindent\textbf{Intuition}.
We use an error decomposition in \citet{bousquet2020sharper} to handle this. Our intuitive \emph{idea} is to show that
\[
\big\|\nabla F(A(S))-\nabla F_S(A(S))\big\|_2\leq 2\epsilon+\frac{1}{n}\big\|\sum_{i=1}^{n}\xi_i\big\|_2,
\]
where $\xi_i$ is a sequence of mean-zero variables satisfying $\ebb[\langle\xi_i,\xi_j\rangle]\leq 4\epsilon^2$ for any $i\neq j$. Then one can show that
\[
\frac{1}{n^2}\ebb[\big\|\sum_{i=1}^{n}\xi_i\big\|_2^2]\leq\frac{1}{n^2}\sum_{i=1}^{n}\ebb[\|\xi_i\|_2^2]+4\epsilon^2=O(1/n+\epsilon^2).
\]
\begin{proof}[Proof of Theorem \ref{thm:stab-gen-grad}]
Let $S,S'$ and $S^{(i)}$ be defined as in the proof of Lemma \ref{lem:gen-stab}.
We have the following error decomposition
\begin{multline*}
  n\big(\nabla F(A(S))-\nabla F_S(A(S))\big) = \sum_{i=1}^{n}\ebb_{Z,z_i'}\Big[\nabla f(A(S);Z)-\nabla f(A(S^{(i)});Z)\Big]+\\
  \sum_{i=1}^{n}\!\ebb_{z_i'}\Big[\ebb_Z[\nabla f(A(S^{(i)});Z)]\!-\!\nabla f(A(S^{(i)});z_i)\Big]
  +\sum_{i=1}^{n}\ebb_{z_i'}\Big[\nabla f(A(S^{(i)});z_i)-\nabla f(A(S);z_i)\Big],
\end{multline*}
where we have used $\ebb_Z[\nabla f(A(S);Z)]=\nabla F(A(S))$.
It then follows that
\begin{multline*}
  n\big\|\nabla F(A(S))-\nabla F_S(A(S))\big\|_2 \leq
  \sum_{i=1}^{n}\ebb_{Z,z_i'}\Big[\|\nabla f(A(S);Z)-\nabla f(A(S^{(i)});Z)\|_2\Big]\\
  +\Big\|\sum_{i=1}^{n}\xi_i(S)\Big\|_2 +\sum_{i=1}^{n}\ebb_{z_i'}\Big[\Big\|\nabla f(A(S^{(i)});z_i)-\nabla f(A(S);z_i)\Big\|_2\Big],
\end{multline*}
where we introduce $\xi_i$ as a function of $S$ as follows
\[
\xi_i(S)=\ebb_{z_i'}\Big[\ebb_Z[\nabla f(A(S^{(i)});Z)]-\nabla f(A(S^{(i)});z_i)\Big],\forall i\in[n].
\]
Note $S$ and $S^{(i)}$ differ by a single example.
By the assumption on stability, we further get
\begin{equation}\label{stab-gen-grad-0}
n\ebb\big[\big\|\nabla F(A(S))-\nabla F_S(A(S))\big\|_2\big]\leq 2n\epsilon+\ebb\Big[\big\|\sum_{i=1}^{n}\xi_i(S)\big\|_2\Big].
\end{equation}
Due to the symmetry between $Z$ and $z_i$, one can see that
\begin{equation}\label{mean-zero}
\ebb_{z_i}[\xi_i(S)]=0,\quad\forall i\in[n].
\end{equation}
Introduce $S''=\{z_1'',\ldots,z_n''\}$ which are drawn independently from $\rho$. For each $i,j\in[n]$ with $i\neq j$, introduce
\begin{gather*}
  S_j=\{z_1,\ldots,z_{j-1},z_j'',z_{j+1},\ldots,z_n\}, \\
   S^{(i)}_j=\{z_1,\ldots,z_{i-1},z_i',z_{i+1},\ldots,z_{j-1},z_j'',z_{j+1},\ldots,z_n\}.
\end{gather*}
That is, $S_j$ is formed by replacing the $j$-th element of $S$ with $z_j''$, while $S^{(i)}_j$ is formed by replacing the $j$-th element of $S^{(i)}$ with $z_j''$.
If $i\neq j$, then
\[
\ebb\big[\langle\xi_i(S_j),\xi_j(S)\rangle\big]=\ebb\ebb_{z_j}\big[\langle\xi_i(S_j),\xi_j(S)\rangle\big]
=\ebb\Big[\langle\xi_i(S_j),\ebb_{z_j}[\xi_j(S)]\rangle\Big]=0,
\]
where the second identity holds since $\xi_i(S_j)$ is independent of $z_j$ and the last identity follows from $\ebb_{z_j}\big[\xi_j(S)]=0$ due to \eqref{mean-zero}. In a similar way, one can show the following inequalities for $i\neq j$
\[
\ebb\big[\langle\xi_i(S),\xi_j(S_i)\rangle\big]=\ebb\ebb_{z_i}\big[\langle\xi_i(S),\xi_j(S_i)\rangle\big]
=\ebb\Big[\langle\xi_j(S_i),\ebb_{z_i}[\xi_i(S)]\rangle\Big]=0
\]
and
\begin{align*}
\ebb\big[\langle\xi_i(S_j),\xi_j(S_i)\rangle\big]&=\ebb\ebb_{z_i}\big[\langle\xi_i(S_j),\xi_j(S_i)\rangle\big]
=\ebb\big[\langle\xi_j(S_i),\ebb_{z_i}[\xi_i(S_j)]\rangle\big]=0.
\end{align*}
As a combination of the above identities we have ($i\neq j$)
\begin{align}
   \ebb\big[\langle\xi_i(S),\xi_j(S)\rangle\big]
   & = \ebb\Big[\big\langle\xi_i(S)-\xi_i(S_j),\xi_j(S)-\xi_j(S_i)\big\rangle\Big] \notag\\
   & \leq \ebb\Big[\big\|\xi_i(S)-\xi_i(S_j)\big\|_2\big\|\xi_j(S)-\xi_j(S_i)\big\|_2\Big] \notag\\
   & \leq \frac{1}{2}\ebb\Big[\big\|\xi_i(S)-\xi_i(S_j)\big\|_2^2\Big]+\frac{1}{2}\ebb\Big[\big\|\xi_j(S)-\xi_j(S_i)\big\|_2^2\Big],\label{stab-gen-grad-1}
\end{align}
where we have used $ab\leq\frac{1}{2}(a^2+b^2)$.
According to the definition of $\xi_i(S)$ and $\xi_i(S_j)$ we know the following identity for $i\neq j$
\begin{multline*}
  \ebb\big[\big\|\xi_i(S)-\xi_i(S_j)\big\|_2^2\big]
  = \ebb\Big[ \Big\|\ebb_{z_i'}\ebb_Z\big[\nabla f(A(S^{(i)});Z)-\nabla f(A(S^{(i)}_j);Z)\big]\\
  +\ebb_{z_i'}\big[\nabla f(A(S^{(i)}_j);z_i)-\nabla f(A(S^{(i)});z_i)\big]\Big\|_2^2\Big].
\end{multline*}
It then follows from the elementary inequality $(a+b)^2\leq2(a^2+b^2)$ and the Jensen's inequality that
\begin{multline*}
  \ebb\big[\big\|\xi_i(S)-\xi_i(S_j)\big\|_2^2\big]\leq 2\ebb\Big[ \Big\|\nabla f(A(S^{(i)});Z)- \nabla f(A(S^{(i)}_j);Z)\Big\|_2^2\Big]\\
  +2\ebb\Big[\Big\|\nabla f(A(S^{(i)}_j);z_i)-\nabla f(A(S^{(i)});z_i)\Big\|_2^2\Big].
\end{multline*}
Since $S^{(i)}$ and $S^{(i)}_j$ differ by one example, it follows from the definition of stability that
\[
\ebb\big[\big\|\xi_i(S)-\xi_i(S_j)\big\|_2^2\big]\leq 4\epsilon^2,\qquad\forall i\neq j.
\]
In a similar way, one can show that
\[
\ebb\Big[\big\|\xi_j(S)-\xi_j(S_i)\big\|_2^2\Big]\leq 4\epsilon^2,\qquad\forall i\neq j.
\]
We can plug the above two inequalities back into \eqref{stab-gen-grad-1} and derive the following inequality if $i\neq j$
\[
\ebb\big[\langle\xi_i(S),\xi_j(S)\rangle\big]\leq 4\epsilon^2.
\]
Furthermore, according to the definition of $\xi_i(S)$ and Jensen inequality we know
\begin{align*}
   \ebb\big[\|\xi_i(S)\|_2^2\big]
   & = \ebb\Big[\Big\|\ebb_{z_i'}\Big[\ebb_Z[\nabla f(A(S^{(i)});Z)]-\nabla f(A(S^{(i)});z_i)\Big]\Big\|_2^2\Big] \\
   & \leq \ebb\Big[\Big\|\ebb_Z[\nabla f(A(S^{(i)});Z)]-\nabla f(A(S^{(i)});z_i)\Big\|_2^2\Big] \\
   & = \ebb\Big[\Big\|\ebb_Z[\nabla f(A(S);Z)]-\nabla f(A(S);z_i')\Big\|_2^2\Big] \\
   & = \ebb_S\Big[\mathbb{V}_Z(\nabla f(A(S);Z))\Big],
\end{align*}
where we have used the symmetry between $z_i$ and $z_i'$ ($z_i'$ has the same distribution of $Z$).
It then follows that
\begin{align*}
\ebb\big[\big\|\sum_{i=1}^{n}\xi_i(S)\big\|_2^2\big]
& =\ebb\big[\sum_{i=1}^{n}\|\xi_i(S)\|_2^2\big]+\sum_{i,j\in[n]:i\neq j}\ebb[\langle\xi_i(S),\xi_j(S)\rangle]\\
&\leq n\ebb_S\Big[\mathbb{V}_Z(\nabla f(A(S);Z))\Big]+4n(n-1)\epsilon^2.
\end{align*}
We can plug the above inequality back into \eqref{stab-gen-grad-0} and get
\[
n\ebb\big[\big\|\nabla F(A(S))-\nabla F_S(A(S))\big\|_2\big]\leq 2n\epsilon
+\sqrt{n\ebb_S\Big[\mathbb{V}_Z(\nabla f(A(S);Z))\Big]}+2n\epsilon.
\]
The proof is completed.
\end{proof}

\subsection{Proofs of Theorem \ref{thm:gen-argument} and Theorem \ref{thm:gen-argument-hp}\label{sec:proof-gen-argument}}
In this section, we provide the proof of Theorem \ref{thm:gen-argument} and Theorem \ref{thm:gen-argument-hp}.

\noindent\textbf{Intuition}.
Before giving the detailed proof, we first sketch the intuition. For any $S$, define
\begin{equation}\label{alg:weak-erm}
  \bw_S=\arg\min_{\bv\in\rbb^d} \big\{F_S(\bv)+\rho\|\bv-A(S)\|_2^2\big\},
\end{equation}
\begin{equation}\label{alg:tilde}
  \tilde{\bw}_S=\arg\min_{\bv\in\rbb^d} \big\{F(\bv)+\rho\|\bv-A(S)\|_2^2\big\}.
\end{equation}
According to the definition of $F_{S,1/(2\rho)}$ and $F_{1/(2\rho)}$, we know
\begin{align*}
  \nabla F_{S,1/(2\rho)}(A(S))&=2\rho\big(A(S)-\bw_S\big),\\ 
  \nabla F_{1/(2\rho)}(A(S))&=2\rho\big(A(S)-\tilde{\bw}_S\big). 
\end{align*}
Then we know
\begin{equation}\label{weak-0}
  \nabla F_{S,1/(2\rho)}(A(S))-\nabla F_{1/(2\rho)}(A(S))=2\rho(\tilde{\bw}_S-\bw_S).
\end{equation}
It remains to control $\|\tilde{\bw}_S-\bw_S\|_2$.
According to the definition of $\bw_S$, we know
\begin{equation}\label{idea-1}
  F_S(\bw_S)+\rho\|\bw_S-A(S)\|_2^2\leq F_S(\tilde{\bw}_S)+\rho\|\tilde{\bw}_S-A(S)\|_2^2.
\end{equation}
Let $A$ be $\epsilon$-uniformly argument stable. We then show that the algorithm defined in Eq. \eqref{alg:weak-erm} is $O(\epsilon+1/(n\rho))$-uniformly stable, and the algorithm defined in Eq. \eqref{alg:tilde} is $O(\epsilon)$-uniformly stable. It then follows from the connection between generalization and stability that
\begin{gather*}
\ebb[F(\bw_S)-F_S(\bw_S)]=O(\epsilon+1/(n\rho)),\\
 \ebb[F_S(\tilde{\bw}_S)-F(\tilde{\bw}_S)]=O(\epsilon).
\end{gather*}
We then can replace $F_S$ in Eq. \eqref{idea-1} by $F$ to get
\[
\big(\ebb[F(\bw_S)+\rho\|\bw_S-A(S)\|_2^2]\big)-\big( \ebb[F(\tilde{\bw}_S)+\rho\|\tilde{\bw}_S-A(S)\|_2^2]\big)=O(\epsilon+1/(n\rho)).
\]
Furthermore, the weak-convexity and the optimality of $\tilde{\bw}_S$ show that the left-hand side of the above inequality is larger than $\frac{\rho}{2}\ebb[\|\bw_S-\tilde{\bw}_S\|_2^2]$. We then get the desired bound
$\ebb[\|\bw_S-\tilde{\bw}_S\|_2]=O(\sqrt{\epsilon/\rho}+1/(\sqrt{n}\rho))$.

We now give the detailed proof. We first introduce two lemmas.
Lemma \ref{lem:weak-erm} shows the argument stability of the algorithm $S\mapsto\prox_{F_S/(2\rho)}(A(S))$ via the argument stability of $A$. For any $g:\wcal\mapsto\rbb$, let $\partial g(\bw)$ denote the subdifferential of $g$ at $\bw$.
\begin{lemma}\label{lem:weak-erm}
Let $A$ be an algorithm.
Assume for any $S$, the function $F_S$ is $\rho$-weakly-convex.
For any $S$, let $\bw_S$ be defined in Eq. \eqref{alg:weak-erm} and assume $\sup_{z}\|\nabla f(\bw_S;z)\|_2\leq G$.
Let $S$ and $S'$ be neighboring datasets. Then
  \[
  \|\bw_S-\bw_{S'}\|_2\leq \frac{2G}{\rho n}+2\|A(S)-A(S')\|_2.
  \]
\end{lemma}
\begin{proof}
  Without loss of generality, we assume $S$ and $S'$ differ by the last element, i.e., $S=\{z_1,\ldots,z_n\}$ and $S=\{z_1,\ldots,z_{n-1},z_n'\}$.
  Since $F_S$ is $\rho$-weakly convex, we know
  \begin{equation}\label{weak-erm-1}
  \langle\bw_S-\bw_{S'},\partial F_S(\bw_S) - \partial F_S(\bw_{S'})\rangle\geq -\rho\|\bw_S-\bw_{S'}\|_2^2.
  \end{equation}
  According to the first-order optimality condition we know
  \[
  -2\rho\big(\bw_S-A(S)\big)\in \partial F_S(\bw_S)
  \]
  and
  \[
  -2\rho\big(\bw_{S'}-A(S')\big)\in\partial F_{S'}(\bw_{S'})=
  \partial F_S(\bw_{S'}) + \frac{1}{n}\partial f(\bw_{S'};z_n') - \frac{1}{n}\partial f(\bw_{S'};z_n),
  \]
  where we have used the addition property of subdifferential and the definition of $F_S, F_S'$. We can plug the above two expressions into Eq \eqref{weak-erm-1} and get
  \begin{equation*}
  \Big\langle\bw_S-\bw_{S'},-2\rho(\bw_S-A(S))+2\rho(\bw_{S'}-A(S'))+ \frac{1}{n}\partial f(\bw_{S'};z_n') - \frac{1}{n}\partial f(\bw_{S'};z_n)\Big\rangle \geq -\rho\|\bw_S-\bw_{S'}\|_2^2.
  \end{equation*}
  It then follows from the Lipschitz continuity that
  \begin{align*}
     \rho\|\bw_S-\bw_{S'}\|_2^2
     & \leq \Big\langle\bw_S-\bw_{S'},2\rho(A(S)-A(S'))+ \frac{\partial f(\bw_{S'};z_n')-\partial f(\bw_{S'};z_n)}{n}\Big\rangle\\
     & \leq \|\bw_S\!-\!\bw_{S'}\|_2\Big\|2\rho(A(S)\!-\!A(S'))\!+\! \frac{\partial f(\bw_{S'};z_n')\!-\!\partial f(\bw_{S'};z_n)}{n}\Big\|_2\\
     & \leq \|\bw_S-\bw_{S'}\|_2\Big(2\rho\|A(S)-A(S')\|_2+ \frac{2G}{n}\Big).
  \end{align*}
  The stated bound then follows. The proof is completed.
\end{proof}
The following lemma connects the argument stability of the algorithm $S\mapsto\prox_{F/(2\rho)}(A(S))$ via that of $A$.
\begin{lemma}\label{lem:stab-tilde}
Let $A$ be an algorithm and $F$ be $\rho$-weakly-convex.
For any $S$, let $\tilde{\bw}_S$ be defined in Eq. \eqref{alg:tilde}.
Let $S$ and $S'$ be neighboring datasets. Then
  \[
  \|\tilde{\bw}_S-\tilde{\bw}_{S'}\|_2\leq 2\|A(S)-A(S')\|_2.
  \]
\end{lemma}
\begin{proof}
  By the weak convexity of $F$ we know
  \[
  \langle\tilde{\bw}_S-\tilde{\bw}_{S'},\partial F(\tilde{\bw}_S) - \partial F(\tilde{\bw}_{S'})\rangle\geq -\rho\|\tilde{\bw}_S-\tilde{\bw}_{S'}\|_2^2.
  \]
  According to the first-order optimality condition we know
  \begin{gather*}
  -2\rho\big(\tilde{\bw}_S-A(S)\big)\in \partial F(\tilde{\bw}_S),\\
  -2\rho\big(\tilde{\bw}_{S'}-A(S')\big)\in \partial F(\tilde{\bw}_{S'}).
  \end{gather*}
  As a combination of the above three inequalities, we get
  \[
  \Big\langle\tilde{\bw}_S-\tilde{\bw}_{S'},2\rho\big(\tilde{\bw}_{S'}-A(S')\big)-2\rho\big(\tilde{\bw}_S-A(S)\big)\Big\rangle
  \geq -\rho\|\tilde{\bw}_S-\tilde{\bw}_{S'}\|_2^2.
  \]
  It then follows from the Lipschitz continuity that
  \begin{align*}
    \rho\|\tilde{\bw}_S-\tilde{\bw}_{S'}\|_2^2 & \leq \Big\langle\tilde{\bw}_S-\tilde{\bw}_{S'},2\rho A(S)-2\rho A(S')\Big\rangle \\
     & \leq 2\rho\|\tilde{\bw}_S-\tilde{\bw}_{S'}\|_2\|A(S)-A(S')\|_2.
  \end{align*}
  The stated inequality then follows directly. 
\end{proof}
\begin{proof}[Proof of Theorem \ref{thm:gen-argument}]
For any $S$, define $\bw_S$ and $\tilde{\bw}_S$ according to Eq. \eqref{alg:weak-erm} and Eq. \eqref{alg:tilde}, respectively.
According to Lemma \ref{lem:stab-tilde} and the Lipschitz continuity assumption ($A$ is $\epsilon$-argument stable), we know that the algorithm defined by \eqref{alg:tilde} is $2G\epsilon$-uniformly stable in function values. It then follows from Lemma \ref{lem:gen-stab} that
\begin{equation}\label{weak-1}
  \ebb\big[F_S(\tilde{\bw}_S)-F(\tilde{\bw}_S)\big] \leq 2G\epsilon.
\end{equation}
According to Lemma \ref{lem:weak-erm}, we know that the algorithm defined by \eqref{alg:weak-erm} is $(\frac{2G^2}{n\rho}+2G\epsilon)$-uniformly stable. It then follows from Lemma \ref{lem:gen-stab} that
\[
\ebb[F(\bw_S)-F_S(\bw_S)]\leq \frac{2G^2}{n\rho}+2G\epsilon.
\]
It then follows  that
\begin{multline}
  \ebb[F(\bw_S)+\rho\|\bw_S-A(S)\|_2^2]-\ebb[F_S(\bw_S)+\rho\|\bw_S-A(S)\|_2^2]\\
  =\ebb[F(\bw_S)-F_S(\bw_S)]
  \leq \frac{2G^2}{n\rho}+2G\epsilon.\label{weak-2}
\end{multline}
Furthermore, according to the definition of $\bw_S$ we know
\[
F_S(\bw_S)+\rho\|\bw_S-A(S)\|_2^2\leq F_S(\tilde{\bw}_S)+\rho\|\tilde{\bw}_S-A(S)\|_2^2
\]
and therefore it follows from \eqref{weak-1} that
\begin{align*}
\ebb[F_S(\bw_S)+\rho\|\bw_S-A(S)\|_2^2] & \leq \ebb[F_S(\tilde{\bw}_S)+\rho\|\tilde{\bw}_S-A(S)\|_2^2]\\
& \leq \ebb[F(\tilde{\bw}_S)+\rho\|\tilde{\bw}_S-A(S)\|_2^2]+2G\epsilon.
\end{align*}
We can combine \eqref{weak-2} and the above inequality together, and derive
\[
\ebb[F(\bw_S)+\rho\|\bw_S-A(S)\|_2^2]-\ebb[F(\tilde{\bw}_S)+\rho\|\tilde{\bw}_S-A(S)\|_2^2]
\leq \frac{2G^2}{n\rho}+4G\epsilon.
\]
According to the $\rho$-strong convexity of $\bv\mapsto F(\bv)+\rho\|\bv-A(S)\|_2^2$ (this strong convexity follows from the weak convexity of $F$) and the definition of $\tilde{\bw}_S$ as a minimizer, we know
\[
\ebb[F(\bw_S)+\rho\|\bw_S-A(S)\|_2^2]-\ebb[F(\tilde{\bw}_S)+\rho\|\tilde{\bw}_S-A(S)\|_2^2]
\geq \frac{\rho}{2}\ebb[\|\bw_S-\tilde{\bw}_S\|_2^2].
\]
We can combine the above two inequalities together and derive
\[
\frac{\rho}{2}\ebb[\|\bw_S-\tilde{\bw}_S\|_2^2]\leq \frac{2G^2}{n\rho}+4G\epsilon.
\]
It then follows that
\begin{equation}\label{weak-3}
\ebb[\|\bw_S-\tilde{\bw}_S\|_2]\leq \frac{2G}{\sqrt{n}\rho}+\sqrt{8G\epsilon/\rho}.
\end{equation}
It then follows from Eq. \eqref{weak-0} that
\[
  \ebb\big[\big\|\nabla F_{S,1/(2\rho)}(A(S))-\nabla F_{1/(2\rho)}(A(S))\big\|_2\big]
   = 2\rho\ebb\big[\big\|\bw_S-\tilde{\bw}_S\big\|_2\big]
   \leq \frac{4G}{\sqrt{n}}+\sqrt{32G\epsilon\rho}.
\]
The proof is completed.
\end{proof}

\section{Proof of Theorem \ref{thm:gen-argument-hp}}
In this section, we prove the high probability bounds. To this aim, we first introduce a useful lemma.
\begin{lemma}[\citealt{bousquet2020sharper}\label{lem:bousquet}]
  Let $A$ be an $\epsilon$-uniformly stable algorithm. Assume $f(A(S);z)\leq R$ almost surely. Then for any $\delta\in(0,1)$ with probability at least $1-\delta$ we have
  \[
  \big|F_S(A(S))-F(A(S))\big|\!=\!O\Big(\epsilon\log(n)\log(1/\delta)\!+\!R\sqrt{n^{-1}\log(1/\delta)}\Big).
  \]
\end{lemma}
\begin{proof}[Proof of Theorem \ref{thm:gen-argument-hp}]
For any $S$, define $\bw_S$ and $\tilde{\bw}_S$ according to Eq. \eqref{alg:weak-erm} and Eq. \eqref{alg:tilde}, respectively.
According to Lemma \ref{lem:stab-tilde} and the Lipschitz continuity assumption, we know that the algorithm defined by \eqref{alg:tilde} is $2G\epsilon$-uniformly stable in function values. The following inequality then follows from Lemma \ref{lem:bousquet} with probability at least $1-\delta/2$
\begin{equation}\label{weak-4}
  F_S(\tilde{\bw}_S)-F(\tilde{\bw}_S) = O\Big(\epsilon\log(n)\log(1/\delta)+\sqrt{n^{-1}\log(1/\delta)}\Big).
\end{equation}
According to Lemma \ref{lem:weak-erm}, we know that the algorithm defined by \eqref{alg:weak-erm} is $\big(\frac{2G^2}{n\rho}+2G\epsilon\big)$-uniformly stable. The following inequality then follows from Lemma \ref{lem:bousquet} with probability at least $1-\delta/2$
\[
F(\bw_S)-F_S(\bw_S)=O\Big(\big(G^2(n\rho)^{-1}+G\epsilon\big)\log(n)\log(1/\delta)+\sqrt{n^{-1}\log(1/\delta)}\Big).
\]
It then follows  that
\begin{multline}
  \big(F(\bw_S)+\rho\|\bw_S-A(S)\|_2^2\big)-\big(F_S(\bw_S)+\rho\|\bw_S-A(S)\|_2^2\big)\\ 
  =O\Big(\big(G^2(n\rho)^{-1}+G\epsilon\big)\log(n)\log(1/\delta)+\sqrt{n^{-1}\log(1/\delta)}\Big).\label{weak-5}
\end{multline}
Furthermore, according to the definition of $\bw_S$ and \eqref{weak-4} we know
\begin{align*}
F_S(\bw_S)+\rho\|\bw_S-A(S)\|_2^2 &\leq F_S(\tilde{\bw}_S)+\rho\|\tilde{\bw}_S-A(S)\|_2^2\\
& \leq F(\tilde{\bw}_S)+\rho\|\tilde{\bw}_S-A(S)\|_2^2
+O\Big(\epsilon\log(n)\log(1/\delta)+\sqrt{n^{-1}\log(1/\delta)}\Big).
\end{align*}
We can combine Eq. \eqref{weak-5} and the above inequality together, and derive the following inequality with probability at least $1-\delta$
\begin{multline*}
\big(F(\bw_S)+\rho\|\bw_S-A(S)\|_2^2\big)-\big(F(\tilde{\bw}_S)+\rho\|\tilde{\bw}_S-A(S)\|_2^2\big)\\
=O\Big(\big(G^2(n\rho)^{-1}+G\epsilon\big)\log(n)\log(1/\delta)+\sqrt{n^{-1}\log(1/\delta)}\Big).
\end{multline*}
According to the $\rho$-strong convexity of $\bv\mapsto F(\bv)+\rho\|\bv-A(S)\|_2^2$ and the definition of $\tilde{\bw}_S$, we know the following inequality 
\[
\big(F(\bw_S)+\rho\|\bw_S-A(S)\|_2^2\big)-\big(F(\tilde{\bw}_S)+\rho\|\tilde{\bw}_S-A(S)\|_2^2\big)
\geq \frac{\rho}{2}\|\bw_S-\tilde{\bw}_S\|_2^2.
\]
We can combine the above two inequalities together and derive the following inequality with probability at least $1-\delta$
\[
\frac{\rho}{2}\|\bw_S-\tilde{\bw}_S\|_2^2=
O\Big(\big(G^2(n\rho)^{-1}+G\epsilon\big)\log(n)\log(1/\delta)+\sqrt{n^{-1}\log(1/\delta)}\Big),
\]
from which we derive
\[
\|\bw_S-\tilde{\bw}_S\|_2=O\Big(\big(Gn^{-\frac{1}{2}}\rho^{-1}+\sqrt{G\epsilon/\rho}\big)\sqrt{\log(n)\log(1/\delta)}+\big(n^{-1}\rho^{-2}\log(1/\delta)\big)^{\frac{1}{4}}\Big).
\]
The stated bound then follows from Eq. \eqref{weak-0}.
The proof is completed.
\end{proof}
\section{Proofs on Uniform Stability Bounds\label{sec:proof-stab-gen}}
In this section, we present the proofs on the uniform stability bounds of sampling-determined algorithms. Our proof follows the idea in \citet{hardt2016train}.
\begin{proof}[Proof of Theorem \ref{thm:sampling-determined}]
  Let $S=\{z_1,\ldots,z_n\}$ and $S'=\{z_1',\ldots,z_n'\}$. Without loss of generality, we assume $S$ and $S'$ differ only by the last example, i.e., $z_n\neq z_n'$.
  Let $I(A)=\{i_1,\ldots,i_T\}$ be the set of indices selected in the implementation of $A$.
  We first prove Part (a). According to the property of conditional expectation, we know
  \begin{multline*}
     \ebb_A\big[f(A(S);z)-f(A(S');z)\big]
     = \ebb_A\big[f(A(S);z)-f(A(S');z)|n\not\in I(A)\big]\mathrm{Pr}\{n\not\in I(A)\}\\+\ebb_A\big[f(A(S);z)-f(A(S');z)|n\in I(A)\big]\mathrm{Pr}\{n\in I(A)\}.
  \end{multline*}
  Since $A$ is a sampling-determined algorithm, $A(S)$ is independent of $z_n$ under the condition $n\not\in I(A)$. Therefore, under the condition $n\not\in I(A)$ we have
  $A(S)=A(S')$.
  Therefore,
  \begin{align*}
  \ebb_A\big[f(A(S);z)-f(A(S');z)\big] & = \ebb_A\big[f(A(S);z)-f(A(S');z)|n\in I(A)\big]\mathrm{Pr}\{n\in I(A)\}\\
  & \leq 2B\mathrm{Pr}\{n\in I(A)\},
  \end{align*}
  where we have used the assumption $\ebb_A\big[f(A(S);z)|n\in I(A)\big]\leq B$ for any $S$.

  We now turn to Part (b). It is clear
  \begin{multline*}
     \ebb_A\big[\|\nabla f(A(S);z)-\nabla f(A(S');z)\|_2^2\big]
     = \ebb_A\big[\|\nabla f(A(S);z)-\nabla f(A(S');z)\|_2^2|n\not\in I(A)\big]\mathrm{Pr}\{n\not\in I(A)\}\\+\ebb_A\big[\|\nabla f(A(S);z)-\nabla f(A(S');z)\|_2^2|n\in I(A)\big]\mathrm{Pr}\{n\in I(A)\}.
  \end{multline*}
  It then follows that
  \begin{align*}
  \ebb_A\big[\|\nabla f(A(S);z)-\nabla f(A(S');z)\|_2^2\big] & = \ebb_A\big[\|\nabla f(A(S);z)-\nabla f(A(S');z)\|_2^2|n\in I(A)\big]\mathrm{Pr}\{n\in I(A)\}\\
  & \leq 4G^2\mathrm{Pr}\{n\in I(A)\},
  \end{align*}
  where we have used the assumption $\ebb_A[\|\nabla f(A(S);z)|n\in I(A)\|_2^2]\leq G^2$ for any $S$.

  Finally, we consider Part (c). It is clear
  \begin{align*}
     &\ebb_A\big[\|A(S)-A(S')\|_2\big]\\
     &= \ebb_A\big[\|A(S)-A(S')\|_2|n\not\in I(A)\big]\mathrm{Pr}\{n\not\in I(A)\}+\ebb_A\big[\|A(S)-A(S')\|_2|n\in I(A)\big]\mathrm{Pr}\{n\in I(A)\}\\
     &=\ebb_A\big[\|A(S)-A(S')\|_2|n\in I(A)\big]\mathrm{Pr}\{n\in I(A)\}
     \leq 2R\mathrm{Pr}\{n\in I(A)\},
  \end{align*}
  where we have used the assumption $\ebb_A[\|A(S)\|_2|n\in I(A)]\leq R$ for any $S$.
  The proof is completed.
\end{proof}
\begin{proof}[Proof of Corollary \ref{cor:gen}]
We consider only SGD (the arguments of AdaGrad-Norm are the same).
It is clear that
  \[
  \mathrm{Pr}\{n\in I(A)\}\leq \sum_{t=1}^{T}\mathrm{Pr}\{i_t=n\}\leq \frac{T}{n}.
  \]
It is clear that the algorithm $A$ is sampling-determined, and therefore one can apply Theorem \ref{thm:sampling-determined} to derive the stated bounds.
  The proof is completed.
\end{proof}

\section{Proofs on Stochastic Gradient Descent\label{sec:proof-sgd}}
The following lemma establishes the optimization error bounds of SGD. Part (a) is a standard result in optimization. Part (b) is due to \citet{ghadimi2013stochastic}, Part (c) is due to \citet{vaswani2019fast} and Part (d) is due to \citet{davis2019stochastic}.
\begin{lemma}[Optimization Error Bound for SGD]\label{lem:opt-sgd}
Let $\{\bw_t\}_t$ be produced by SGD and \[\ebb_A[\|\nabla f(\bw_t;z_{i_t})\|_2^2]\leq G^2,\quad\forall t\in[T].\]
\begin{enumerate}[label=(\alph*)]
\item If $F_S$ is convex, then for all $t\in\nbb$ and $\bw$
    \[
    \ebb_A\Big[F_S\Big(\frac{\sum_{t=1}^{T}\eta_t\bw_t}{\sum_{t=1}^{T}\eta_t}\Big)\Big]-F_S(\bw)\leq \frac{G^2\sum_{t=1}^{T}\eta_t^2+\|\bw\|_2^2}{2\sum_{t=1}^{T}\eta_t}.
    \]
\item If for any $z$, the function $\bw\mapsto f(\bw;z)$ is $L$-smooth, then
    \[
    \sum_{t=1}^{T}\eta_t\ebb_A\big[\|\nabla F_S(\bw_t)\|_2^2\big]\leq F_S(\bw_1)+\frac{LG^2}{2}\sum_{t=1}^{T}\eta_t^2.
    \]
\item Assume for all $z$, the function $\bw\mapsto f(\bw;z)$ is $L$-smooth and SGC holds with the parameter $\rho$. If  $\eta_t=1/(\rho L)$, then
\[
\sum_{t=1}^{T}\ebb_A\big[\|\nabla F_S(\bw_t)\|_2^2\big]\leq 2\rho Lf(\bw_1).
\]
\item If $F_S$ is $\rho$-weakly convex, then
\[
\sum_{t=1}^{T}\eta_t\ebb_A\big[\|\nabla F_{S,1/(2\rho)}(\bw_t)\|_2^2\big]=O\Big(1+G^2\rho\sum_{t=1}^{T}\eta_t^2\Big).
\]
\end{enumerate}
\end{lemma}
\begin{proof}[Proof of Proposition \ref{prop:sgd}]
  According to Lemma \ref{lem:opt-sgd}, Part (a), we have the following optimization error bounds
  \[
   \ebb_A[F_S(\bar{\bw}_T)]-F_S(\bw^*)=O\Big(\frac{T\eta^2G^2+\|\bw^*\|_2^2}{T\eta}\Big).
  \]
  Furthermore, by Corollary \ref{cor:gen}, Part (a), we have the following stability bounds
  \[
  \sup_z\ebb\big[f(\bar{\bw}_T;z)-f(\bar{\bw}_T';z)\big]\leq \frac{2BT}{n},
  \]
  where $\{\bw_t'\}$ is a sequence of iterates produced by SGD based on a neighboring dataset $S'$.
  This together with Lemma \ref{lem:gen-stab} on the connection between uniform stability and generalization further implies
  \[
  \ebb\big[F(\bar{\bw}_T)-F_S(\bar{\bw}_T)\big]=O\big(BT/n\big).
  \]
  We can plug the above generalization error and optimization error bounds into \eqref{decomposition}, and derive \eqref{sgd-a}.

  If $\eta\asymp \frac{\|\bw^*\|_2}{Gn^{\frac{1}{3}}}$ and $T\asymp \frac{n^{\frac{2}{3}}G\|\bw^*\|_2}{B}$, we have
  \begin{gather*}
  G^2\eta\asymp \frac{G\|\bw^*\|_2}{n^{\frac{1}{3}}},\quad T\eta\asymp \frac{n^{\frac{2}{3}}G\|\bw^*\|_2}{B}\frac{\|\bw^*\|_2}{Gn^{\frac{1}{3}}}=\frac{n^{\frac{1}{3}}\|\bw^*\|_2^2}{B},\\
  \frac{BT}{n}\asymp \frac{n^{\frac{2}{3}}G\|\bw^*\|_2}{n}=\frac{G\|\bw^*\|_2}{n^{\frac{1}{3}}}.
  \end{gather*}
  The bound $\ebb[F(\bar{\bw}_T)]-F(\bw^*)=O((B+G\|\bw^*\|_2)n^{-\frac{1}{3}})$ follows directly from the choice of $\eta$ and $T$. The proof is completed.
\end{proof}

\begin{proof}[Proof of Proposition \ref{prop:sgd-nonconvex}]
According to Lemma \ref{lem:opt-sgd}, Part (b), we have the following optimization error bounds
\[
\ebb_A\big[\|\nabla F_S(\bw_r)\|_2^2\big]=O\Big(\frac{T\eta^2G^2+1}{T\eta}\Big)
\]
and therefore
\begin{equation}\label{sgd-nonconvex-1}
\ebb_A\big[\|\nabla F_S(\bw_r)\|_2\big]=O\big(G\sqrt{\eta}+1/\sqrt{T\eta}\big).
\end{equation}
It is clear that $A$ is sampling-determined and one can apply Corollary \ref{cor:gen}, Part (b) to show the following uniform stability bounds
\[
\sup_z\ebb\big[\|\nabla f(\bw_r;z)-\nabla f(\bw'_r;z)\|_2^2\big]\leq \frac{4G^2T}{n},
\]
where $\{\bw'_t\}$ is a sequence of iterates produced by SGD based on a neighboring dataset $S'$.
This together with \eqref{decomposition-grad} and the connection between uniform stability and generalization established in Theorem \ref{thm:stab-gen-grad} gives
\[
\ebb\big[\|\nabla F(\bw_r)\|_2\big] \leq 8G\sqrt{T/n}+G/\sqrt{n} +  \ebb\big[\|\nabla F_S(\bw_r)\|_2\big].
\]
We can plug the optimization error\ bounds \eqref{sgd-nonconvex-1} into the above bound, and get
\[
\ebb\big[\|\nabla F(\bw_r)\|_2\big] = O\Big(\frac{G\sqrt{T}}{\sqrt{n}} +  \frac{G\sqrt{T}\eta+1}{\sqrt{T\eta}}\Big).
\]
If we choose $\eta\asymp 1/(G\sqrt{T})$, we get
\[
\ebb\big[\|\nabla F(\bw_r)\|_2\big] = O\Big(\frac{G\sqrt{T}}{\sqrt{n}} +  \frac{\sqrt{G}}{T^{\frac{1}{4}}}\Big).
\]
We can choose $T\asymp n^{\frac{2}{3}}/G^{\frac{2}{3}}$ to derive the stated bound $\ebb\big[\|\nabla F(\bw_r)\|_2\big]=O(G^{\frac{2}{3}}n^{-\frac{1}{6}})$.
\end{proof}
\begin{proof}[Proof of Proposition \ref{prop:sgd-sgc}]
Analogous to the proof of Proposition \ref{prop:sgd-nonconvex}, we have
\[
\ebb\big[\|\nabla F(\bw_r)\|_2\big] \leq 8G\sqrt{T/n}+G/\sqrt{n} +  \ebb\big[\|\nabla F_S(\bw_r)\|_2\big].
\]
Furthermore, Lemma \ref{lem:opt-sgd}, Part (c) implies
\[
\ebb\big[\|\nabla F_S(\bw_r)\|_2\big]=O(\sqrt{L\rho}/\sqrt{T}).
\]
We can combine the above two bounds together and get
\[
\ebb\big[\|\nabla F(\bw_r)\|_2\big]=O\Big(G\sqrt{T/n}+\sqrt{L\rho}/\sqrt{T}\Big).
\]
Therefore, we can choose $T\asymp \sqrt{L\rho n}/G$ and get $\ebb\big[\|\nabla F(\bw_r)\|_2\big]=O((L\rho G^2/n)^{\frac{1}{4}})$. The proof is completed.
\end{proof}

\begin{proof}[Proof of Proposition \ref{prop:sgd-wc}]
According to Lemma \ref{lem:opt-sgd}, Part (d), we have the following optimization error bounds
\[
\ebb_A\big[\|\nabla F_{S,1/(2\rho)}(\bw_r)\|_2^2\big]=O\Big(\frac{\rho TG^2\eta^2+1}{T\eta}\Big)
\]
and therefore
\begin{equation}\label{sgd-nonconvex-1}
\ebb_A\big[\|\nabla F_{S,1/(2\rho)}(\bw_r)\|_2\big]=O\big(G\sqrt{\rho\eta}+1/\sqrt{T\eta}\big).
\end{equation}
We can apply Corollary \ref{cor:gen}, Part (c) to show the following argument stability bounds
\[
\ebb_A\big[\|A(S)-A(S')\|_2\big]\leq \frac{2RT}{n}.
\]
This together with \eqref{decomposition-envelope} and the connection between argument stability and generalization established in Theorem \ref{thm:gen-argument} gives
\begin{equation}\label{proof-stab-gen-argument}
\ebb\big[\big\|\nabla F_{1/(2\rho)}(A(S))\big\|_2\big]\leq
\ebb\big[\big\|\nabla F_{S,1/(2\rho)}(A(S))\big\|_2\big]+
\frac{4G}{\sqrt{n}}+\sqrt{64GRT\rho n^{-1}}.
\end{equation}
We can plug the optimization error bounds \eqref{sgd-nonconvex-1} into the above bound, and get
\[
\ebb\big[\big\|\nabla F_{1/(2\rho)}(A(S))\big\|_2\big] = O\Big(G\sqrt{\rho\eta}+\sqrt{GR\rho T/n}+1/\sqrt{T\eta}+G/\sqrt{n}\Big).
\]
If we choose $\eta\asymp 1/(G\sqrt{\rho T})$, we get
\[
\ebb\big[\|\nabla F(\bw_r)\|_2\big] = O\Big(\sqrt{G}(\rho/T)^{\frac{1}{4}} +  \sqrt{GR\rho T/n}+G/\sqrt{n}\Big).
\]
We can choose $T\asymp  n^{\frac{2}{3}}/(R^{\frac{2}{3}}\rho^{\frac{1}{3}})$ to derive the stated bound $\ebb\big[\|\nabla F(\bw_r)\|_2\big]=O(\sqrt{G}\rho^{\frac{1}{3}}R^{\frac{1}{6}}n^{-\frac{1}{6}})$.
The proof is completed.
\end{proof}

\section{AdaGrad-Norm\label{sec:adagrad}}
\vspace*{-0.066cm}
\subsection{Generalization Bounds of AdaGrad-Norm}
\vspace*{-0.066cm}

We now turn to the generalization analysis of AdaGrad-Norm. Proposition \ref{prop:adagrad} presents the risk bounds in terms of function values for convex and nonsmooth problems, while Proposition \ref{prop:adagrad-nonconvex} presents the risk bounds in terms of gradients for nonconvex and smooth problems. Note that these bounds match the corresponding results for SGD (w.r.t. $n$) in Section \ref{sec:sgd} up to a logarithmic factor. All the proofs are given in Section \ref{sec:proof-adagrad}.
\begin{proposition}[Convex and Nonsmooth Case]\label{prop:adagrad}
Let $\{\bw_t\}_t$ be produced by \eqref{adagrad}, $\ebb[\|\nabla f(\bw_t;z_{i_t})\|_2^2]\leq G^2$ for all $t\in[T]$ and $\sup_{\bw\in\wcal}\|\bw\|_2\leq R$. Let $A$ output $\bar{\bw}_T=\frac{1}{T}\sum_{t=1}^{T}\bw_t$. If $F_S$ is convex we have
  \[
  \ebb_{S,A}[F(\bar{\bw}_T)]-F(\bw^*)=O\big(GRT/n\big)+O(G(R+\|\bw^*\|_2)/\sqrt{T}).
  \]
  If $T\asymp n^{\frac{2}{3}}$ we have
  \[
  \ebb_{S,A}[F(\bar{\bw}_T)]-F(\bw^*)=O(G(R+\|\bw^*\|_2)n^{-\frac{1}{3}}).
  \]
\end{proposition}
\begin{proposition}[Nonconvex and Smooth Case]\label{prop:adagrad-nonconvex}
Let $\{\bw_t\}_t$ be produced by \eqref{adagrad} and $\ebb[\|\nabla f(\bw_t;z_{i_t})\|_2^2]\leq G^2$ for all $t\in[T]$.
If $A(S)=\bw_r$ and $F_S$ is $L$-smooth, then
\[
\ebb_{S,A,r}\big[\|\nabla F(\bw_r)\|_2\big] = O\Big(G\sqrt{T/n}+GT^{-\frac{1}{4}}\log^{\frac{1}{2}} T\Big).
\]
If $T\asymp n^{\frac{2}{3}}$, one gets
\[
\ebb\big[\|\nabla F(\bw_r)\|_2\big]=O(Gn^{-\frac{1}{6}}\log^{\frac{1}{2}}n).
\]
\end{proposition}
\begin{remark}\normalfont
  Generalization behavior of adaptive gradient descent was recently studied by \citet{zhou2020towards}. They considered minibatch adaptive algorithms with a sufficiently large batch size, while the algorithms we consider here use only a single example to compute a stochastic gradient and is therefore more computationally efficient. Their analysis is based on a connection between generalization and differential privacy, and requires to add noise to achieve differential privacy. This in turn leads to a dimension-dependent bound. As a comparison, we do not require to introduce noise in algorithms and our bounds are dimension-free.
\end{remark}

\subsection{Proofs on AdaGrad-Norm\label{sec:proof-adagrad}}
The following lemma establishes the convergence rates of AdaGrad-Norm. Part (a) is for convex and nonsmooth problems, while Part (b) is for nonconvex and smooth problems. We give a simple proof of Part (a), while the proof of Part (b) can be found in \citet{ward2020adagrad}.
\begin{lemma}[Optimization Error Bound for AdaGrad-Norm\label{lem:opt-adagrad}]
  Let $\{\bw_t\}$ be the sequence produced by AdaGrad-Norm.
  \begin{enumerate}[label=(\alph*)]
    \item Let $F_S$ be convex. Assume $\|\bw\|_2\leq R$ for all $\bw\in\wcal$. Then the following bound holds for all $\bw\in\wcal$
    \[
    \ebb_A\big[F_S(\bar{\bw}_T)\big]-F_S(\bw)=O(G(R+\|\bw\|_2)/\sqrt{T}).
    \]
    \item Assume $F_S$ is $L$-smooth, $\ebb\big[\|\nabla f(\bw_t;z_{i_t})\|_2^2\big]\leq G^2$ for all $\bw_t$. Then
    \[
    \ebb_{A,r}\big[\|\nabla F_S(\bw_r)\|_2\big] = O\Big(GT^{-\frac{1}{4}}\log^{\frac{1}{2}} T\Big),
    \]
    where $r$ follows from the uniform distribution over $[T]$.
  \end{enumerate}
\end{lemma}
\begin{proof}
  Denote $\eta_t=\eta/b_t$, then \eqref{adagrad} can be written as $\bw_{t+1}=\Pi_{\wcal}(\bw_t-\eta_t\nabla f(\bw_t;z_{i_t}))$. It then follows that
  \[
  \|\bw_{t+1}-\bw\|_2^2 \leq \|\bw_t-\bw\|_2^2 - 2\eta_t\langle\bw_t-\bw,\nabla f(\bw_t;z_{i_t})\rangle + \eta_t^2\|\nabla f(\bw_t;z_{i_t})\|_2^2.
  \]
  Re-arranging the above inequality gives
  \[
  \langle\bw_t-\bw,\nabla f(\bw_t;z_{i_t})\rangle \leq \frac{1}{2\eta_t}\Big(\|\bw_t-\bw\|_2^2-
  \|\bw_{t+1}-\bw\|_2^2\Big) + \frac{\eta_t}{2}\|\nabla f(\bw_{t};z_{i_t})\|_2^2.
  \]
  We take conditional expectation w.r.t. $z_{i_t}$ over both sides and get
  \[
  \langle\bw_t-\bw,\nabla F_S(\bw_t)\rangle \leq \ebb_{i_t}\Big[\frac{1}{2\eta_t}\Big(\|\bw_t-\bw\|_2^2-
  \|\bw_{t+1}-\bw\|_2^2\Big)\Big] + \ebb_{i_t}\Big[\frac{\eta_t}{2}\|\nabla f(\bw_{t};z_{i_t})\|_2^2\Big].
  \]
  It then follows from the convexity of $F_S$ that
  \[
  F_S(\bw_t)-F_S(\bw)\leq \ebb_{i_t}\Big[\frac{1}{2\eta_t}\Big(\|\bw_t-\bw\|_2^2-
  \|\bw_{t+1}-\bw\|_2^2\Big)\Big] + \ebb_{i_t}\Big[\frac{\eta_t}{2}\|\nabla f(\bw_{t};z_{i_t})\|_2^2\Big].
  \]
  We can take an expectation followed with a summation of the above inequality from $t=1$ to $t=T$, and get
  \begin{align*}
     & \sum_{t=1}^{T}\ebb_A\big[F_S(\bw_t)-F_S(\bw)\big] - \ebb_A\big[\frac{1}{2\eta_1}\|\bw_1-\bw\|_2^2\big] \\
     & \leq \frac{1}{2}\sum_{t=2}^{T}\ebb_A\Big[\|\bw_t-\bw\|_2^2\Big(\frac{1}{\eta_t}-\frac{1}{\eta_{t-1}}\Big)\Big]
     +\frac{1}{2}\sum_{t=1}^{T}\ebb_A\Big[\eta_t\|\nabla f(\bw_t;z_{i_t})\|_2^2\Big] \\
     & \leq (R^2+\|\bw\|_2^2)\sum_{t=2}^{T}\ebb_A\Big[\frac{1}{\eta_t}-\frac{1}{\eta_{t-1}}\Big]
     +\frac{\eta}{2}\sum_{t=1}^{T}\ebb_A\bigg[\frac{\|\nabla f(\bw_t;z_{i_t})\|_2^2}{\sqrt{\sum_{\tau=1}^{t}\|\nabla f(\bw_\tau;z_{i_\tau})\|_2^2}}\bigg] \\
     & \leq (R^2+\|\bw\|_2^2)\eta^{-1}\ebb_A\Big[\Big(\sum_{t=1}^{T}\|\nabla f(\bw_\tau;z_{i_t})\|_2^2\Big)^{\frac{1}{2}}\Big]+\eta\ebb_A\Big[\Big(\sum_{t=1}^{T}\|\nabla f(\bw_\tau;z_{i_t})\|_2^2\Big)^{\frac{1}{2}}\Big],
  \end{align*}
  where we have used the following inequality in the last step
  \[
  \sum_{t=1}^{T}\frac{a_t}{\sqrt{\sum_{j=1}^{t}a_j}} \leq \sum_{t=1}^{T}\int_{\sum_{j=1}^{t-1}a_j}^{\sum_{j=1}^{t}a_j}\frac{1}{\sqrt{x}}dx=\int_0^{\sum_{j=1}^{T}a_j}\frac{1}{\sqrt{x}}dx= 2\sqrt{\sum_{t=1}^{T}a_t}.
  \]
  It then follows from the convexity of $F_S$ that
  \[
  \ebb_A\Big[F_S(\bar{\bw}_T)-F_S(\bw)\Big]= O\Big(\frac{1}{T}\ebb_A\Big[\Big(\sum_{t=1}^{T}\|\nabla f(\bw_\tau;z_{i_t})\|_2^2\Big)^{\frac{1}{2}}\Big]
  \Big((R^2+\|\bw\|_2^2)\eta^{-1}+\eta\Big)\Big).
  \]
  The stated bound then follows.
\end{proof}

\begin{proof}[Proof of Proposition \ref{prop:adagrad}]
  Analogous to the proof of Proposition \ref{prop:sgd}, we have the following generalization error bound
  \[
  \ebb\big[F(\bar{\bw}_T)-F_S(\bar{\bw}_T)\big]=O\big(GRT/n\big).
  \]
  Lemma \ref{lem:opt-adagrad}, Part (a), implies the following optimization error bound
  \[
    \ebb_A\big[F_S(\bar{\bw}_T)\big]-F_S(\bw^*)=O(G(R+\|\bw^*\|_2)/\sqrt{T}).
  \]
  We can plug the above two inequalities back into \eqref{decomposition} and get
  \[
  \ebb\big[F(\bar{\bw}_T)\big]-F(\bw^*)=O\big(GRT/n\big)+O(G(R+\|\bw^*\|_2)/\sqrt{T}).
  \]
  One can derive the stated bound by setting $T\asymp n^{\frac{2}{3}}$. The proof is completed.
\end{proof}

\begin{proof}[Proof of Proposition \ref{prop:adagrad-nonconvex}]
Analogous to the Proof of Proposition \ref{prop:sgd-nonconvex}, we have the following generalization error bound
\[
\ebb\big[\|\nabla F(\bw_r)\|_2\big] \leq 8G\sqrt{T/n}+G/\sqrt{n} +  \ebb\big[\|\nabla F_S(\bw_r)\|_2\big].
\]
Lemma \ref{lem:opt-adagrad}, Part (b), implies the following optimization error bound
\[
\ebb_{A,r}\big[\|\nabla F_S(\bw_r)\|_2\big] = O\Big(GT^{-\frac{1}{4}}\log^{\frac{1}{2}} T\Big).
\]
We can combine the above two inequalities together and get
\[
\ebb\big[\|\nabla F(\bw_r)\|_2\big]=O\Big(G\sqrt{T/n}+GT^{-\frac{1}{4}}\log^{\frac{1}{2}} T\Big).
\]
One can choose $T\asymp n^{\frac{2}{3}}$ to get the stated bound. The proof is completed.
\end{proof}

\section{Differentially Private SGD\label{sec:dp-sgd}}
\subsection{Utility and Privacy Guarantee}
In this section, we use our stability analysis to develop a differentially private SGD with generalization guarantee for weakly-convex problems, which is useful to handle data with sensitive information~\citep{dwork2008differential}. We first introduce the definition of \emph{differential privacy}, which is a well-accepted mathematical definition of privacy.
\begin{definition}[Differential Privacy]
Let $\epsilon>0$ and $\delta\in(0,1)$.
A randomized mechanism $\acal$ provides $(\epsilon,\delta)$-differential privacy (DP) if for any two neighboring datasets $S$ and $S'$ and any set $E$ in the range of $\acal$ there holds
\[
\pbb(\acal(S)\in E)\leq e^\epsilon\pbb(\acal(S')\in E)+\delta.
\]
\end{definition}

\begin{algorithm2e}
\caption{\small Differentially Private SGD\label{alg:dp-sgd}}
\DontPrintSemicolon 
\KwIn{$\bw_1=0$, learning rates $\{\eta_t\}_t$, parameter $\beta,\epsilon,\delta>0$ and dataset $S=\{z_1,\ldots,z_n\}$}
\For{$t=1,2,\ldots,T$}{
compute $\sigma$ by Eq. \eqref{sigma}\\
draw $i_t$ uniformly from $[n]$ and $b_t\sim \ncal(0,\sigma^2\ibb_{d})$ \\
update $\bw_{t+1}$ according to Eq. \eqref{dp-sgd}
}
\KwOut{$\bw_r$ where $r\sim\mathrm{unif}[T]$}
\end{algorithm2e}


Our basic idea to develop differentially private algorithms is to inject noise in the learning process to mask the influence of any single datapoint. In particular, at the $t$-th iteration we randomly sample a noise $b_t$ from a Gaussian distribution with a variance $\sigma^2\ibb_{d}$ and build a new stochastic gradient as $\nabla f(\bw_t;z_{i_t})+b_t$. Then we move along the negative direction of this stochastic gradient as follows
\begin{equation}\label{dp-sgd}
  \bw_{t+1}=\Pi_{\wcal}\big(\bw_t-\eta_t(\nabla f(\bw_t;z_{i_t})+b_t)\big),
\end{equation}
where $\beta$ is a parameter and
\begin{equation}\label{sigma}
  \sigma^2=\frac{14G^2T}{\beta n^2\epsilon}\Big(\frac{\log(1/\delta)}{(1-\beta)\epsilon}+1\Big).
\end{equation}
We refer to our algorithm as DP-SGD and summarize the implementation in Algorithm \ref{alg:dp-sgd}. Proposition \ref{prop:privacy-dp-sgd} shows that Algorithm \ref{alg:dp-sgd} achieves the $(\epsilon,\delta)$-privacy guarantee, while Proposition \ref{prop:utility-dp-sgd} gets the utility guarantee as measured by $\|\nabla F(\bw_r)\|_2$. The proofs are given in Section \ref{sec:proof-dp-sgd}.
\begin{proposition}[Privacy guarantee\label{prop:privacy-dp-sgd}]
  Let $\epsilon>0$ and $\delta\in(0,1)$. Assume for any $z$, the function $\bw\mapsto f(\bw;z)$ is $G$-Lipschitz. If
  \begin{equation}\label{privacy-dp-sgd-condition}
  \epsilon\geq\frac{14T}{3n^2}\quad\text{and}\quad\frac{\log(1/\delta)}{\epsilon}\leq \frac{\sqrt{n}}{3\sqrt{3}}-\frac{5}{3},
  \end{equation}
  then we can choose $\beta=\frac{7T}{3n^2\epsilon}$ and Algorithm \ref{alg:dp-sgd} satisfies $(\epsilon,\delta)$-DP.
\end{proposition}
\begin{proposition}[Utility guarantee\label{prop:utility-dp-sgd}]
Let $\epsilon>0$ and $\delta\in(0,1)$. Assume for any $z$, the function $\bw\mapsto f(\bw;z)$ is $G$-Lipschitz and $F_S$ is $\rho$-weakly convex.
Let $\{\bw_t\}_t$ be produced by Algorithm \ref{alg:dp-sgd} with $\eta_t=\eta$ and $A(S)=\bw_r$. Assume $\ebb_{S,A}[\|\nabla f(\bw_t;z_{i_t})\|_2^2]\leq G^2$ and $\ebb_A[\|A(S)\|_2|n\in I(A)]\leq R$.
Let Eq. \eqref{privacy-dp-sgd-condition} hold and $\beta=\frac{7T}{3n^2\epsilon}$. If we choose $\eta\asymp1/((G+\sqrt{d}\sigma)\sqrt{\rho T})$ and $T\asymp (1+d^{\frac{1}{3}}G^{\frac{2}{3}}\log^{\frac{2}{3}}(1/\delta)\epsilon^{-\frac{2}{3}})n^{\frac{2}{3}}/(R^{\frac{2}{3}}\rho^{\frac{1}{3}})$, then
\begin{equation}\label{utility-dp-sgd}
\ebb\big[\|\nabla F_{1/(2\rho)}(\bw_r)\|_2\big]=
O(\sqrt{G}R^{\frac{1}{6}}\rho^{\frac{1}{3}}(1+d^{\frac{1}{6}}G^{\frac{1}{3}}\log^{\frac{1}{3}}(1/\delta)\epsilon^{-\frac{1}{3}})n^{-\frac{1}{6}}).
\end{equation}
\end{proposition}

\subsection{Proofs on Differentially Private SGD\label{sec:proof-dp-sgd}}

In this section, we prove privacy and utility guarantee for DP-SGD. To this aim, we first study the R\'enyi differential privacy~\citep{mironov2017renyi}, and then transform it to $(\epsilon,\delta)$-DP. 
\begin{definition}
  For $\lambda>1,\rho>0$, a randomized mechanism $\acal$ satisfies $(\lambda,\rho)$-R\'enyi differential privacy (RDP) if for all neighboring datasets $S$ and $S'$ we have
  \[
  D_\lambda(\acal(S)\|\acal(S')):=\frac{1}{\lambda-1}\log\int\Big(\frac{P_{\acal(S)}(\bw)}{P_{\acal(S')}(\bw)}\Big)^\lambda dP_{\acal(S')}(\bw)\leq\rho,
  \]
  where $P_{\acal(S)}(\bw)$ and $P_{\acal(S')}(\bw)$ are the density of $\acal(S)$ and $\acal(S')$, respectively.
\end{definition}
We first introduce some necessary lemmas.
The following lemma establishes the RDP of a Gaussian mechanism together with subsampling~\citep{liang2020exploring}.
\begin{lemma}[\citealt{liang2020exploring}\label{lem:gauss-mec}]
  Consider a mechanism $\mcal:\zcal^m\mapsto\rbb^d$ and let $\Delta$ be its $\ell_2$-sensitivity, i.e., $\Delta=\sup_{S\sim S'}\|\mcal(S)-\mcal(S')\|_2$. The Gaussian mechanism $\acal=\mcal+\ncal(0,\sigma^2\ibb_d)$ applied to a subset of samples that are drawn uniformly without replacement with subsampling rate $p$ satisfies $(\lambda,3.5p^2\lambda\Delta^2/\sigma^2)$-RDP if
  \[
  \sigma^2\geq0.67\Delta^2\quad\text{and}\quad \lambda-1\leq\frac{2\sigma^2}{3\Delta^2}\log\Big(\frac{1}{\lambda p(1+\sigma^2/\Delta^2)}\Big).
  \]
\end{lemma}
The following lemma shows the RDP of an adaptive composition of several mechanisms.
\begin{lemma}[\citealt{mironov2017renyi}\label{lem:rdp-adaptive}]
  If $\acal_1,\ldots,\acal_k$ are randomized algorithms satisfying, respectively, $(\alpha,\epsilon_1)$-RDP,\ldots,$(\alpha,\epsilon_k)$-RDP, then their composition defined as $(\acal_1(S),\ldots,\acal_k(S))$ is $(\alpha,\epsilon_1+\ldots,+\epsilon_k)$-RDP. Moreover, the $i$th algorithm can be chosen on the basis of the outputs of $\acal_1,\ldots,\acal_{i-1}$.
\end{lemma}
The following lemma shows the connection between DP and RDP.
\begin{lemma}[\citealt{mironov2017renyi}\label{lem:rdp-dp}]
  If a randomized mechanism $\acal$ satisfies $(\lambda,\rho)$-RDP, then $\acal$ satisfies $(\rho+\log(1/\delta)/(\lambda-1),\delta)$-DP for all $\delta\in(0,1)$.
\end{lemma}

We are now ready to prove the privacy and utility guarantee.
\begin{proof}[Proof of Proposition \ref{prop:privacy-dp-sgd}]
  Consider the mechanism $\acal_t=\mcal_t+b_t$, where $\mcal_t(z)=\nabla f(\bw_t;z)$. Since $f$ is $G$-Lipschitz continuous, we know
  \[
  \sup_{z,z'}\|\nabla f(\bw_t;z)-\nabla f(\bw_t;z')\|_2\leq\Delta:= 2G
  \]
  and therefore the $\ell_2$ sensitivity of $\mcal_t$ is $2G$.
  Note
  \[
  \frac{\sigma^2}{\Delta^2}=\frac{14G^2T}{\beta n^2\epsilon}\Big(\frac{\log(1/\delta)}{(1-\beta)\epsilon}+1\Big)\frac{1}{4G^2}=\frac{7T}{2\beta n^2\epsilon}\Big(\frac{\log(1/\delta)}{(1-\beta)\epsilon}+1\Big).
  \]
  According to Lemma \ref{lem:gauss-mec}, we know $\mcal_t$ satisfies $\Big(\lambda,\frac{\lambda\beta\epsilon}{T\big(\frac{\log(1/\delta)}{(\beta-1)\epsilon}+1\big)}\Big)$-RDP if
  \begin{equation}\label{privacy-dp-sgd-1}
  \frac{7T}{2\beta n^2\epsilon}\Big(\frac{\log(1/\delta)}{(1-\beta)\epsilon}+1\Big)\geq 0.67
  \end{equation}
  and
  \[
  \lambda-1\leq \frac{7T}{3\beta n^2\epsilon}\Big(\frac{\log(1/\delta)}{(1-\beta)\epsilon}+1\Big)\log\Big(\frac{n}{\lambda\big(1+\frac{7T}{2\beta n^2\epsilon}\big(\frac{\log(1/\delta)}{(1-\beta)\epsilon}+1\big)\big)}\Big).
  \]
  Let $\lambda=\frac{\log(1/\delta)}{(1-\beta)\epsilon}+1$.
  Then the above inequality becomes
  \begin{equation}\label{privacy-dp-sgd-2}
    \frac{\log(1/\delta)}{(1-\beta)\epsilon}\leq \frac{7T}{3\beta n^2\epsilon}\Big(\frac{\log(1/\delta)}{(1-\beta)\epsilon}+1\Big)\log\Big(\frac{n}{\big(\frac{\log(1/\delta)}{(1-\beta)\epsilon}+1\big)\big(1+\frac{7T}{2\beta n^2\epsilon}\big(\frac{\log(1/\delta)}{(1-\beta)\epsilon}+1\big)\big)}\Big).
  \end{equation}
  We first suppose Eq. \eqref{privacy-dp-sgd-1}, \eqref{privacy-dp-sgd-2} hold and prove the stated bound under these conditions.
  With our definition of $\lambda$, we know $\mcal_t$ satisfies $\big(\frac{\log(1/\delta)}{(1-\beta)\epsilon}+1,\frac{\beta\epsilon}{T}\big)$-RDP for any $t\in[T]$. By the adaptive composition (Lemma \ref{lem:rdp-adaptive}), we know Algorithm \ref{alg:dp-sgd} satisfies $\big(\frac{\log(1/\delta)}{(1-\beta)\epsilon}+1,\beta\epsilon\big)$-RDP. It then follows from Lemma \ref{lem:rdp-dp} that Algorithm \ref{alg:dp-sgd} satisfies $(\epsilon,\delta)$-DP. We now show that Eq. \eqref{privacy-dp-sgd-1} and Eq. \eqref{privacy-dp-sgd-2} hold. Since $\epsilon\geq14T/(3n^2)$ we know $\beta\leq1/2$.
  It is clear
  \begin{equation}\label{privacy-dp-sgd-3}
  \frac{7T}{3\beta n^2\epsilon}=\frac{7T\cdot 3n^2\epsilon}{21Tn^2\epsilon}=1\geq0.67.
  \end{equation}
  Therefore, Eq. \eqref{privacy-dp-sgd-1} holds.
  Furthermore, the assumption $\frac{\log(1/\delta)}{\epsilon}\leq \frac{\sqrt{n}}{3\sqrt{3}}-\frac{5}{3}$ implies $1+\frac{3}{2}\big(\frac{2\log(1/\delta)}{\epsilon}+1\big)\leq\frac{\sqrt{n}}{\sqrt{3}}$. It then follows that
  \begin{align*}
    & \Big(\frac{\log(1/\delta)}{(1-\beta)\epsilon}+1\Big)\Big(1+\frac{7T}{2\beta n^2\epsilon}\Big(\frac{\log(1/\delta)}{(1-\beta)\epsilon}+1\Big)\Big) \leq \Big(\frac{2\log(1/\delta)}{\epsilon}+1\Big)\Big(1+
    \frac{7T}{2\beta n^2\epsilon}\Big(\frac{2\log(1/\delta)}{\epsilon}+1\Big)\Big) \\
     & = \Big(\frac{2\log(1/\delta)}{\epsilon}+1\Big)\Big(1+\frac{3}{2}\Big(\frac{2\log(1/\delta)}{\epsilon}+1\Big)\Big)  \leq \Big(1+\frac{3}{2}\Big(\frac{2\log(1/\delta)}{\epsilon}+1\Big)\Big)^2\leq \frac{n}{3}.
  \end{align*}
  We can combine the above inequality and Eq. \eqref{privacy-dp-sgd-3} to show Eq. \eqref{privacy-dp-sgd-2}. The proof is completed.
\end{proof}
\begin{proof}[Proof of Proposition \ref{prop:utility-dp-sgd}]
  Analogous to Lemma \ref{lem:opt-sgd}, Part (d), we have the following optimization error bounds for Algorithm \ref{alg:dp-sgd}
  \[
  \ebb_A\big[\|\nabla F_{S,1/(2\rho)}(\bw_r)\|_2^2\big]=O\Big(\frac{\rho T(G^2+\sigma^2d)\eta^2+1}{T\eta}\Big)
  \]
  and therefore
  \[
  \ebb_A\big[\|\nabla F_{S,1/(2\rho)}(\bw_r)\|_2\big]=O\big((G+\sigma\sqrt{d})\sqrt{\rho\eta}+1/\sqrt{T\eta}\big).
  \]
  Adding noise does not affect the stability analysis~\citep{bassily2020stability}, we then use Eq. \eqref{proof-stab-gen-argument} to get
  \[
  \ebb\big[\big\|\nabla F_{1/(2\rho)}(\bw_r)\big\|_2\big] = O\Big((G+\sigma\sqrt{d})\sqrt{\rho\eta}+\sqrt{GR\rho T/n}+1/\sqrt{T\eta}+G/\sqrt{n}\Big).
  \]
  For our choice of $\beta$, we have
  \begin{equation}\label{sigma-bound}
  \sigma^2= \frac{14G^2T\cdot 3n^2\epsilon}{7T n^2\epsilon}\Big(\frac{\log(1/\delta)}{(1-\beta)\epsilon}+1\Big)\leq 6G^2\Big(\frac{2\log(1/\delta)}{\epsilon}+1\Big),
  \end{equation}
  where we have used $\beta\leq1/2$ established in the proof of Proposition \ref{prop:privacy-dp-sgd}.
  If we choose $\eta\asymp1/((G+\sqrt{d}\sigma)\sqrt{\rho T})$ and use Eq. \eqref{sigma-bound}, we get
  \[
  \ebb\big[\|\nabla F_{1/(2\rho)}(\bw_r)\|_2\big] = O\Big(\sqrt{G+\sqrt{d}\sigma}(\rho/T)^{\frac{1}{4}} +  \sqrt{GR\rho T/n}+G/\sqrt{n}\Big).
  \]
  We can choose $T\asymp (1+d^{\frac{1}{3}}G^{\frac{2}{3}}\log^{\frac{2}{3}}(1/\delta)\epsilon^{-\frac{2}{3}})n^{\frac{2}{3}}/(R^{\frac{2}{3}}\rho^{\frac{1}{3}})$ to get Eq. \eqref{utility-dp-sgd}.
  The proof is completed.
\end{proof}

\section{Convergence Rates with Relaxed Bounded Gradient Assumptions\label{sec:proof-bound-var}}
In this section, we study the convergence rates of SGD for solving weakly convex problems. The existing convergence analysis requires a bounded subgradient assumption as $\ebb_{i_t}[\|\nabla f(\bw_t;z_{i_t})\|_2^2]\leq G^2$ for some $G>0$~\citep{davis2019stochastic}. We aim to relax this assumption to a more general assumption as
\begin{equation}\label{self-bounding}
  \ebb_{i_t}[\|\nabla f(\bw_t;z_{i_t})\|_2^2]\leq B_1\ebb_{i_t}[f(\bw_t;z_{i_t})]+B_2,
\end{equation}
where $B_1,B_2\geq0$ are two constants. This assumption implies that the gradients can be bounded in terms of function values, which has been considered in the literature \citep{zhang2004solving}.
\begin{theorem}\label{thm:bound-var}
Let $\wcal=\rbb^d$ and $\sum_{t=1}^{T}\eta_t^2=O(1)$. Let $\{\bw_t\}_t$ be produced by the algorithm $A$ defined by SGD and Eq. \eqref{self-bounding} holds for all $t\in\nbb$.
If $F_S$ is $\rho$-weakly convex, then
\begin{equation}\label{bound-var-a}
  \sum_{t=1}^{T}\eta_t\ebb_A[\|\nabla F_{S,1/2\rho}(\bw_{t})\|_2^2]=O\Big(1+\sum_{t=1}^{T}\eta_t^2\Big).
\end{equation}
\end{theorem}
\begin{proof}
For any $t\in\nbb$, denote $\hat{\bw}_t=\mbox{Prox}_{F_S,1/(2\rho)}(\bw_t)$.
According to the definition of Moreau envelope and the definition of $\hat{\bw}_t$, we know
\begin{align}
   & \ebb_{i_t}[F_{S,1/2\rho}(\bw_{t+1})] \leq \ebb_{i_t}\big[F_S(\hat{\bw}_t)+\rho\|\hat{\bw}_t-\bw_{t+1}\|_2^2\big] \notag\\
   & = F_S(\hat{\bw}_t) + \rho\ebb_{i_t}\big[\|\hat{\bw}_t-\bw_t+\eta_t\nabla f(\bw_t;z_{i_t})\|_2^2\big] \notag\\
   & = F_S(\hat{\bw}_t)+\rho\|\hat{\bw}_t-\bw_t\|_2^2 + 2\rho\eta_t\ebb_{i_t}\big[\langle\hat{\bw}_t-\bw_t,\nabla f(\bw_t;z_{i_t})\rangle\big]+\rho\eta_t^2\ebb_{i_t}[\|\nabla f(\bw_t;z_{i_t})\|_2^2] \notag\\
   & \leq F_{S,1/(2\rho)}(\bw_t) + 2\rho\eta_t\langle\hat{\bw}_t-\bw_t,\nabla F_S(\bw_t)\rangle + \rho\eta_t^2\ebb_{i_t}\big[B_1f(\bw_t;z_{i_t})+B_2\big] \notag\\
   & \leq F_{S,1/(2\rho)}(\bw_t) + 2\rho\eta_t\big(F_S(\hat{\bw}_t)-F_S(\bw_t)+\frac{\rho}{2}\|\bw_t-\hat{\bw}_t\|_2^2\big) + \rho\eta_t^2\big[B_1F_S(\bw_t)+B_2\big],\label{bound-var-0}
\end{align}
where in the last second step we have used Eq. \eqref{self-bounding} and in the last inequality we have used the weak convexity of $F_S$. By the weak convexity of $F_S$, we know the function $\bw\mapsto F_S(\bw)+\rho\|\bw-\bv\|_2^2$ is $\rho$-strongly convex. This together with the definition of $\hat{\bw}_t$ implies
\begin{align*}
  & F_S(\bw_t)-F_S(\hat{\bw}_t)-\frac{\rho}{2}\|\bw_t-\hat{\bw}_t\|_2^2 \\
  &= \big(F_S(\bw_t)+\rho\|\bw_t-\bw_t\|_2^2\big)
  -\big(F_S(\hat{\bw}_t)+\rho\|\bw_t-\hat{\bw}_t\|_2^2\big)+\frac{\rho}{2}\|\bw_t-\hat{\bw}_t\|_2^2\\
  & \geq\rho\|\bw_t-\hat{\bw}_t\|_2^2.
\end{align*}
It then follows that
\begin{equation}\label{bound-var-1}
  F_S(\bw_t)-F_S(\hat{\bw}_t)\geq\frac{3\rho}{2}\|\bw_t-\hat{\bw}_t\|_2^2\geq0.
\end{equation}
This together with the assumption $\eta_t\leq1/B_1$ implies
\begin{align*}
  B_1\eta_t^2F_S(\bw_t) & = B_1\eta_t^2\big(F_S(\bw_t)-F_S(\hat{\bw}_t)\big) + B_1\eta_t^2F_S(\hat{\bw}_t) \\
   & \leq \eta_t\big(F_S(\bw_t)-F_S(\hat{\bw}_t)\big) + B_1\eta_t^2F_S(\hat{\bw}_t).
\end{align*}
We can plug the above inequality back into Eq. \eqref{bound-var-0} and derive
\begin{align}
   & \ebb_{i_t}[F_{S,1/2\rho}(\bw_{t+1})] \notag \\
   & \leq  F_{S,1/(2\rho)}(\bw_t) + \rho^2\eta_t\|\bw_t-\hat{\bw}_t\|_2^2 + \big(2\rho\eta_t-\rho\eta_t\big)\big(F_S(\hat{\bw}_t)-F_S(\bw_t)\big)+\rho\eta_t^2\big(B_1F_S(\hat{\bw}_t)+B_2\big)\notag\\
   & = F_{S,1/(2\rho)}(\bw_t) +  \rho^2\eta_t\|\bw_t-\hat{\bw}_t\|_2^2 +\rho\eta_t\big(F_S(\hat{\bw}_t)-F_S(\bw_t)\big)+\rho\eta_t^2\big(B_1F_S(\hat{\bw}_t)+B_2\big)\notag\\
   & \leq F_{S,1/(2\rho)}(\bw_t) - \frac{\rho^2\eta_t\|\bw_t-\hat{\bw}_t\|_2^2}{2} + \rho\eta_t^2\big(B_1F_{S,1/(2\rho)}(\bw_t)+B_2\big),\label{bound-var-2}
\end{align}
where we have used Eq. \eqref{bound-var-1} and the following inequality in the last step
\[
F_{S,1/(2\rho)}(\bw_t)=\inf_{\bv}\big\{F_S(\bv)+\rho\|\bv-\bw_t\|_2^2\big\}=F_S(\hat{\bw}_t)+\rho\|\hat{\bw}_t-\bw_t\|_2^2\geq F_S(\hat{\bw}_t).
\]
It then follows from Eq. \eqref{bound-var-2} that
\[
\ebb_A\big[F_{S,1/2\rho}(\bw_{t+1})\big] \leq \big(1+\rho B_1\eta_t^2\big)\ebb_A\big[F_{S,1/2\rho}(\bw_{t})\big] + \rho\eta_t^2B_2
\]
and therefore ($1+a\leq\exp(a)$)
\begin{align*}
  \ebb_A\big[F_{S,1/2\rho}(\bw_{t+1})\big] & \leq \prod_{k=1}^{t}\big(1+\rho B_1\eta_k^2\big)F_{S,1/2\rho}(\bw_{1}) + \rho B_2\sum_{k=1}^{t}\eta_k^2\prod_{\tilde{k}=k+1}^{t}\big(1+\rho B_1\eta_{\tilde{k}}^2\big) \\
  & \leq \prod_{k=1}^{t}\exp(\rho B_1\eta_k^2\big)F_{S,1/2\rho}(\bw_{1}) + \rho B_2\sum_{k=1}^{t}\eta_k^2\prod_{\tilde{k}=k+1}^{t}\exp(\rho B_1\eta_{\tilde{k}}^2\big)\\
  & = \exp\Big(\rho B_1\sum_{k=1}^t\eta_k^2\Big)F_{S,1/2\rho}(\bw_{1}) + \rho B_2\sum_{k=1}^{t}\eta_k^2\exp\Big(\rho B_1\sum_{\tilde{k}=k+1}^{t}\eta_{\tilde{k}}^2\Big).
\end{align*}
Since $\sum_{t=1}^{T}\eta_t^2=O(1)$, we further get
\begin{equation}\label{bound-var-3}
  \ebb_A\big[F_{S,1/2\rho}(\bw_{t})\big]=O(1),\qquad\forall t\in[T].
\end{equation}
We can plug the above inequality back into Eq. \eqref{bound-var-2} and get
\[
\ebb_{A}[F_{S,1/2\rho}(\bw_{t+1})] = \ebb_{A}[F_{S,1/2\rho}(\bw_{t})]- \frac{\rho^2\eta_t\ebb_A[\|\bw_t-\hat{\bw}_t\|_2^2]}{2} + O(\eta_t^2).
\]
The above inequality can be reformulated as
\[
\frac{\rho^2\eta_t\ebb_A[\|\bw_t-\hat{\bw}_t\|_2^2]}{2}=\ebb_{A}[F_{S,1/2\rho}(\bw_{t})]-\ebb_{A}[F_{S,1/2\rho}(\bw_{t+1})]+O(\eta_t^2).
\]
We can take a summation of the above inequality from $t=1$ to $t=T$ and get
\[
\frac{\rho^2}{2}\sum_{t=1}^{T}\eta_t\ebb_A[\|\bw_t-\hat{\bw}_t\|_2^2]=O\Big(1+\sum_{t=1}^{T}\eta_t^2\Big).
\]
According to the definition of $\hat{\bw}_t$, we know $\nabla F_{S,1/2\rho}(\bw_{t})=2\rho(\bw_t-\hat{\bw}_t)$. It then follows that
\[
\sum_{t=1}^{T}\eta_t\ebb_A[\|\nabla F_{S,1/2\rho}(\bw_{t})\|_2^2]=O\Big(1+\sum_{t=1}^{T}\eta_t^2\Big).
\]
This gives the bound \eqref{bound-var-a}. The proof is completed. 
\end{proof}

\end{document}